\documentclass[nohyperref]{article}

\usepackage{microtype}
\usepackage{graphicx}
\usepackage{booktabs} 

\usepackage{hyperref}

\usepackage{icml2022}

\usepackage{amsmath}
\usepackage{amssymb}
\usepackage{mathtools}
\usepackage{amsthm}

\usepackage[capitalize,noabbrev]{cleveref}

\theoremstyle{plain}
\newtheorem{theorem}{Theorem}[section]

\newtheorem{lemma}[theorem]{Lemma}

\theoremstyle{definition}
\newtheorem{definition}[theorem]{Definition}
\newtheorem{assumption}[theorem]{Assumption}
\theoremstyle{remark}

\usepackage{amsfonts}
\usepackage{amssymb,amsthm,amsmath}
\usepackage{esint}
\usepackage{latexsym}
\usepackage{epic}
\usepackage{epsfig}
\usepackage{multirow}
\usepackage{hhline}
\usepackage{wrapfig}
\usepackage{verbatim}
\usepackage{enumitem}
\usepackage{color}
\usepackage{thmtools}

\usepackage{mathtools}
\usepackage{graphicx}
\usepackage{hyperref}
\usepackage{tabu}
\usepackage{colortbl}

\usepackage[labelformat=simple]{subcaption}

\newcommand{\norm}[1]{\left\lVert#1\right\rVert}

\DeclareMathOperator*{\argmin}{arg\,min}

\newcommand{\decision}{\boldsymbol x}
\newcommand{\decisionset}{\mathcal{X}}
\newcommand{\parameter}{\boldsymbol \theta}
\newcommand{\parameterset}{\boldsymbol\Theta}
\newcommand{\error}{\epsilon}

\newcommand{\N}{\mathbb{N}}

\newcommand{\R}{\mathbb{R}}

\newcommand{\E}{\mathbb{E}}

\newcommand{\OPT}{\mathrm{OPT}} 
\newcommand{\ALG}{\mathrm{ALG}} 
\newcommand{\cost}{\mathrm{cost}} 

\usepackage[textsize=tiny]{todonotes}

\icmltitlerunning{Smoothed Online Combinatorial Optimization Using Imperfect Predictions}

\begin{document}

\twocolumn[
\icmltitle{Smoothed Online Combinatorial Optimization Using Imperfect Predictions}



\icmlsetsymbol{equal}{*}

\begin{icmlauthorlist}
\icmlauthor{Firstname1 Lastname1}{equal,yyy}
\icmlauthor{Firstname2 Lastname2}{equal,yyy,comp}
\icmlauthor{Firstname3 Lastname3}{comp}
\icmlauthor{Firstname4 Lastname4}{sch}
\icmlauthor{Firstname5 Lastname5}{yyy}
\icmlauthor{Firstname6 Lastname6}{sch,yyy,comp}
\icmlauthor{Firstname7 Lastname7}{comp}
\icmlauthor{Firstname8 Lastname8}{sch}
\icmlauthor{Firstname8 Lastname8}{yyy,comp}
\end{icmlauthorlist}

\icmlaffiliation{yyy}{Department of XXX, University of YYY, Location, Country}
\icmlaffiliation{comp}{Company Name, Location, Country}
\icmlaffiliation{sch}{School of ZZZ, Institute of WWW, Location, Country}

\icmlcorrespondingauthor{Firstname1 Lastname1}{first1.last1@xxx.edu}
\icmlcorrespondingauthor{Firstname2 Lastname2}{first2.last2@www.uk}

\icmlkeywords{Machine Learning, ICML}

\vskip 0.3in
]



\printAffiliationsAndNotice{\icmlEqualContribution} 

\begin{abstract}
Smoothed online combinatorial optimization considers a learner who repeatedly chooses a combinatorial decision to minimize an unknown changing cost function with a penalty on switching decisions in consecutive rounds. We study smoothed online combinatorial optimization problems when an imperfect predictive model is available, where the model can forecast the future cost functions with uncertainty. We show that using predictions to plan for a finite time horizon leads to regret dependent on the total predictive uncertainty and an additional switching cost. This observation suggests choosing a suitable planning window to balance between uncertainty and switching cost, which leads to an online algorithm with guarantees on the upper and lower bounds of the cumulative regret. Lastly, we provide an iterative algorithm to approximately solve the planning problem in real-time. Empirically, our algorithm shows a significant improvement in cumulative regret compared to other baselines in synthetic online distributed streaming problems.
\end{abstract}

\section{Introduction}
We consider the {\it smoothed online combinatorial optimization} problem, which is an extension of online convex optimization~\cite{hazan2019introduction,shalev2011online,zinkevich2003online,hazan2007logarithmic} and smoothed online convex optimization~\cite{lin2012online,lin2012dynamic}.
In the smoothed online combinatorial optimization problem, an online learner is repeatedly optimizing a cost function with unknown changing parameter. In every time step, the learner chooses a feasible decision from a combinatorial feasible region before observing the parameter of the cost function. After the learner chooses the decision, the learner receives (i) the cost function parameter and the associated cost (ii) an additional known switching cost function dependent on the chosen decision and the previous decision. The goal of the learner is to minimize the cumulative cost in $T$ time steps, including cost produced by the cost function and the switching cost.

Smoothed online combinatorial optimization is commonly seen in applications with online combinatorial decisions and switching penalty, including ride sharing with combinatorial driver-customer assignment~\cite{jia2017optimization}, distributed streaming system with bipartite data-to-server assignment~\cite{garg2013apache,thein2014apache}, and A-B testing in advertisement~\cite{bhat2020near}.
All these examples incur a potential switching cost when the decisions are changed, e.g., reassigning drivers or data to different locations or servers is costly, and changing advertisement campaign requires additional human resources.
The challenge of online combinatorial decision-making and the presence of hidden switching cost motivate the study of smoothed online combinatorial optimization. 

In this paper, we study the smoothed online combinatorial optimization where an imperfect predictive model is available. 
We assume that the predictive model can forecast the future cost parameters with uncertainties, and the uncertainties can evolve over time.
We measure the performance of online algorithms by {\it dynamic regret}, which assumes a dynamic offline benchmark, i.e., the optimal performance when the cost function parameters are given a priori and the sequential decisions are allowed to change.
The same use of predictions and dynamic regret are also studied in receding horizon control~\cite{mattingley2011receding,camacho2013model} in smoothed online convex optimization under different assumptions on the predictions~\cite{chen2015online,badiei2015online,chen2016using,li2020leveraging,li2020online}. 
In our case, the challenges of bounding dynamic regret inherit from smoothed online convex optimization, while the additional combinatorial structure further complicates the analysis.

\subsection*{Main Contribution}
Our main contribution is an online algorithm that plans ahead using the imperfect predictions within a dynamic planning window determined based on the predictive uncertainty of the predictive model.
We summarize our contributions as follows:

\begin{itemize}
    \item Given imperfect predictions with uncertainties, we show that planning based on predictions within a finite time horizon leads to a regret bound that is a function of the total predictive uncertainty with an additional potential switching cost. This bound quantifies one source of regret corresponding to the imperfectness of the predictions, while the other source comes from the additional switching cost (Theorem~\ref{thm:finite-proactive-planning}).
    \item Our regret bound in finite time horizon suggests using a dynamic planning window to optimally balance two sources of regret coming from predictive uncertainty and the switching cost, respectively. Iteratively selecting a dynamic planning window to plan ahead leads to a regret bound in infinite time horizon (Theorem~\ref{thm:dynamic-proactive-planning}).
    \item Specifically, when the uncertainties converge to $0$ when more data is collected, we show that the cumulative regret is always sublinear (Theorem~\ref{thm:regret-bound-full}), which guarantees the no-regretness of Algorithm~\ref{alg:infinite-proactive-planning}. We also quantify the dependency of the cumulative regret on the convergence rate of the uncertainty in some special cases (Corollary~\ref{coro:regret-bound}).
    \item Lastly, we show a lower bound on the total regret for any randomized online algorithm when predictive uncertainty is present. The order of the lower bound matches to the order of the upper bound in some special cases, which guarantees the tightness of our online algorithm and the corresponding regret bounds (Corollary~\ref{coro:lower-bound}).
\end{itemize}





Lastly, given predictions and dynamic planning windows, the smoothed online combinatorial optimization problem reduces to an offline combinatorial problem. We use an iterative algorithm to find an approximate solution to the offline problem efficiently, which largely reduces the computation cost compared to solving the large combinatorial problem using mixed-integer linear program.

Empirically, we evaluate our algorithm on the online distributed streaming problem motivated from Apache Kafka with synthetic traffic. We compare our algorithm using predictions and dynamic planning windows with various baselines. Our algorithm using predictions outperforms baselines without using predictions. Our experiments show an improvement of choosing the right dynamic planning windows against algorithms using fixed planning window, which demonstrates the importance of balancing uncertainty and the switching cost. The use of iterative algorithm also largely reduces the computation cost while keeping a comparable performance, leading to an effective scalable online algorithm that can be applied to real-world problems.

\subsection{Related Work}

\paragraph{Online convex optimization}
Online convex optimization~\cite{gemp2016online,hazan2019introduction,shalev2011online,zinkevich2003online} assumes the objective function is convex and no switching cost.
In online convex optimization, {\it static regret} is most commonly used, which assumes a static benchmark with full information but the decisions over the entire time steps have to be static. 
Various variants of online gradient descents~\cite{zinkevich2003online,hazan2007logarithmic,bartlett2007adaptive,srebro2011universality,flaxman2004online} were proposed with bounds on the static regret.
However, the gradient-based approaches and the regret bounds do not directly generalize to the combinatorial setting due to the discreteness of the feasible region.

\paragraph{Smoothed online convex optimization with predictions}
Smoothed online convex optimization generalizes online convex optimization by assuming a switching cost that defines the cost of moving from the previous decision to the current one.
\cite{andrew2013tale} showed that smoothed online convex optimization can achieve the same static regret bound using the algorithms in online convex optimization without switching cost.
In terms of dynamic regret, receding horizon control~\cite{mattingley2011receding} was proposed to leverage the predictions of future time step to make decision.
Perfect~\cite{lin2012dynamic,lin2012online} and imperfect~\cite{chen2015online,chen2016using,li2020leveraging,li2020online} predictions are used to bound the performance of receding horizon control with fixed planning window size.
Separately, chasing convex bodies~\cite{sellke2020chasing,bubeck2019competitively,bubeck2020chasing,friedman1993convex} shares the same challenge of smoothed online convex optimization but focuses on the competitive ratio.

Nonetheless, the analyses in the convex objectives and feasible regions do not apply to the combinatorial setting. The planning window in receding horizon control is also restricted to be fixed across different time steps.

\paragraph{Online combinatorial optimization and metrical task system}
Online combinatorial optimization assumes a discrete feasible region that the learner can choose from {\it before} seeing the cost function. Existing results~\cite{audibert2014regret,koolen2010hedging} focus on bounding dynamic regret in the case of linear objectives without switching cost.
On the other hand, metrical task system assumes $n$ discrete states that the learner can choose {\it after} seeing the cost function, and there is a metrical switching cost associated to every switch. Existing results focus on bounding {\it competitive ratio}, where the competitive ratio is lower bounded by $\Omega(\frac{\log n}{\log \log n})$~\cite{bartal2006ramsey,bartal2003metric} and upper bounded by $O(\log^2 n)$~\cite{bubeck2021metrical}. 
In contrast, {\it dynamic regret} is a stronger additive guarantee and is more challenging to analyze.

Our work shows that analyzing dynamic regret in an arbitrary smoothed online combinatorial optimization problem becomes tractable when an imperfect predictive model is given.

\section{Problem Statement}
An instance of smoothed online combinatorial optimization is composed of a cost function $f: \decisionset \times \parameterset \rightarrow \R_{\geq 0}$ where $\decision \in \decisionset$ denotes all the feasible decisions that can be taken and $\parameter \in \parameterset$ denotes all the possible unknown parameters of the cost function, and a metric $d: \decisionset \times \decisionset \rightarrow \R_{\geq 0}$ that is used to measure the distance of different decisions.
At each time step $t$, the learner receives a feature $\xi_t \in \Xi$ that is correlated to the unknown parameters in the future.
Based on the given feature $\xi_t$, the learner can predict the future parameters and choose a feasible decision $\decision_t \in \decisionset$ without seeing the future parameter $\parameter_t$. The parameter $\parameter_t$ is revealed after the decision is executed and the learner receives an objective cost $f(\decision_t, \parameter_t)$ with a switching cost $d(\decision_{t-1}, \decision_t)$ which measures the movement of the decisions made by time step $t$ and $t-1$. The total cost of an online algorithm $\ALG$ up to time $T$ is the summation of both the objective cost and the switching cost across all time steps:
\begin{align*}
    \text{cost}_T(\ALG) = \sum\nolimits_{t=1}^T f(\decision_t, \parameter_t) + d(\decision_t, \decision_{t-1}).
\end{align*}

We want to compare to the offline benchmark $\OPT$ in time $T$ that knows all the parameters in advance, which minimizes the total cost defined below:
\begin{align*}
    \text{cost}_T(\OPT) = \min\limits_{\decision_t \in \decisionset ~\forall t} ~\sum\nolimits_{t =1}^T f(\decision_t, \parameter_t) + d(\decision_t, \decision_{t-1})
\end{align*}

\begin{definition}
An online algorithm $\ALG$ has a dynamic regret $\rho(T)$ if we have:
\begin{align*}
\mathrm{Reg}_T \coloneqq \cost_T(\ALG) - \cost_T(\OPT) \leq \rho(T) \quad \forall T.
\end{align*}
\end{definition}

The goal of the learner is to design an online algorithm with a small dynamic regret bound $\rho(T)$.

\subsection{Example: Online Distributed Streaming Systems}\label{sec:experiment}
One application of smoothed online combinatorial optimization problems is the online load balancing problem in the distributed streaming system known as Apache Kafka~\cite{garg2013apache,thein2014apache}. The system is composed of $k$ topics of streaming data and $m$ servers as shown in Figure~\ref{fig:kafka}. At each time step $t$, the system maintains a bipartite assignment $\decision_t$ between $k$ topics and $m$ servers so that the servers can process the streaming data in real time. 
Specifically, each topic must be assigned to exactly one server.
We use $\decision_t \in \decisionset_t \subseteq \{0,1\}^{k \times m}$ with $\decision_{t,i,j} = 1$ to denote assigning the topic $i$ to server $j$ at time $t$. The learner can use the parameters in the prior $H$ time steps as the feature $\xi_t$ that is correlated to the unknown future parameters. After the assignment is chosen, a new traffic vector $\parameter_t \in \R^k$ arrives with each entry representing the number of incoming messages associated to the topic. Figure~\ref{fig:kafka} illustrates how the data-to-server assignment works. A commonly used server imbalance cost is defined as makespan $f(\decision_t, \parameter_t) = \norm{\decision_t^\top \parameter_t}_{\infty}$, the largest load across all servers.

\begin{figure}[t]
    \centering
    \includegraphics[width=0.95\linewidth]{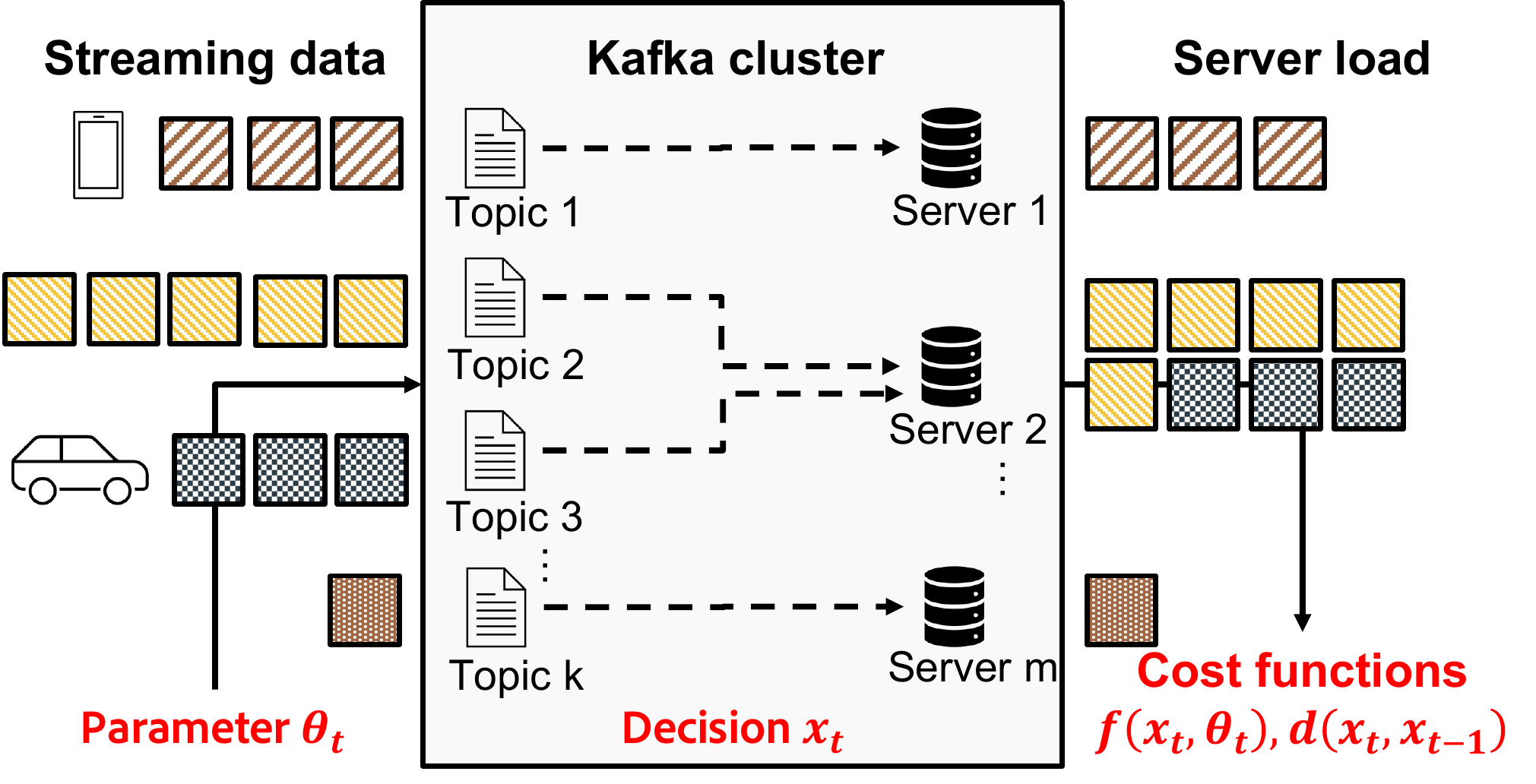}
    \caption{Apache Kafka maintains a bipartite assignment $\decision_t$ between $k$ topics and $m$ servers to prepare for processing the streaming data. The streaming traffic $\parameter_t$ comes later and gets routed to the corresponding servers. A server imbalance cost $f(\decision_t, \parameter_t)$ and a switching cost $d(\decision_t, \decision_{t-1})$ due to assignment change are received.}
    \label{fig:kafka}
\end{figure}

\paragraph{Paper structure}
We first discuss how planning based on predictions works and how to bound the associated dynamic regret using predictive uncertainty.
Second, we discuss two different sources of regret, predictive uncertainty and the number of planning windows used. We propose to use a dynamic planning window to balance the tradeoff with a regret bound derived.
Third, we propose an iterative algorithm to solve an offline problem by decoupling the temporal dependency caused by switching cost.
Lastly, an application in distributed streaming system and Apache Kafka is discussed and used in our experiments.

\begin{figure*}[t]
    \centering
    \begin{subfigure}{0.32\linewidth}
        \centering
        \includegraphics[width=\textwidth]{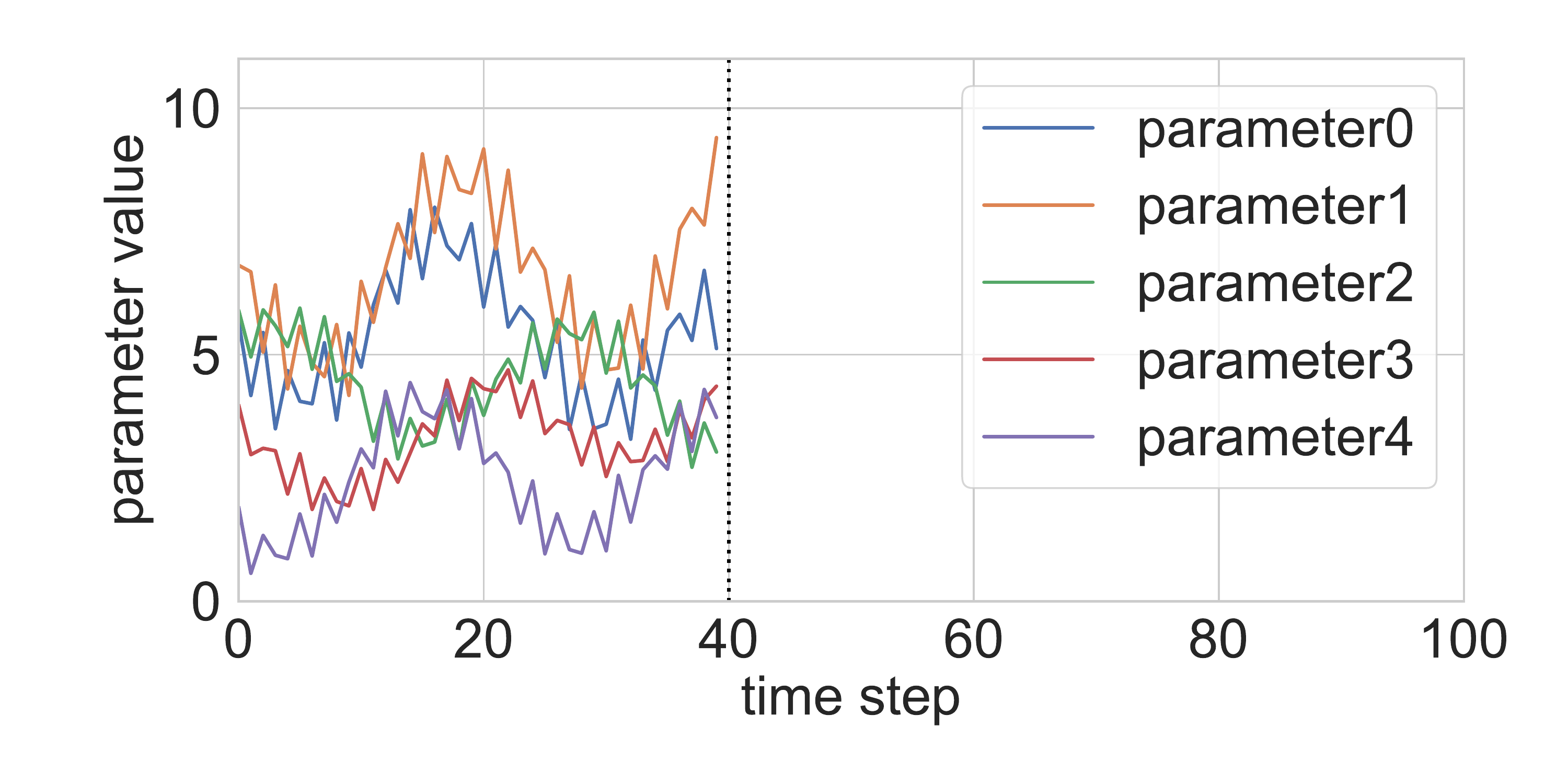}
        \caption{The learner has access to the historical parameters $\{ \parameter_s \in \R^k \}_{s < t}$. We plot each entry of the parameter prior to time $t$ as a time series to visualize the trend.}
        \label{fig:historical-data}
    \end{subfigure}
    \hfill
    \begin{subfigure}{0.32\linewidth}
        \centering
        \includegraphics[width=\textwidth]{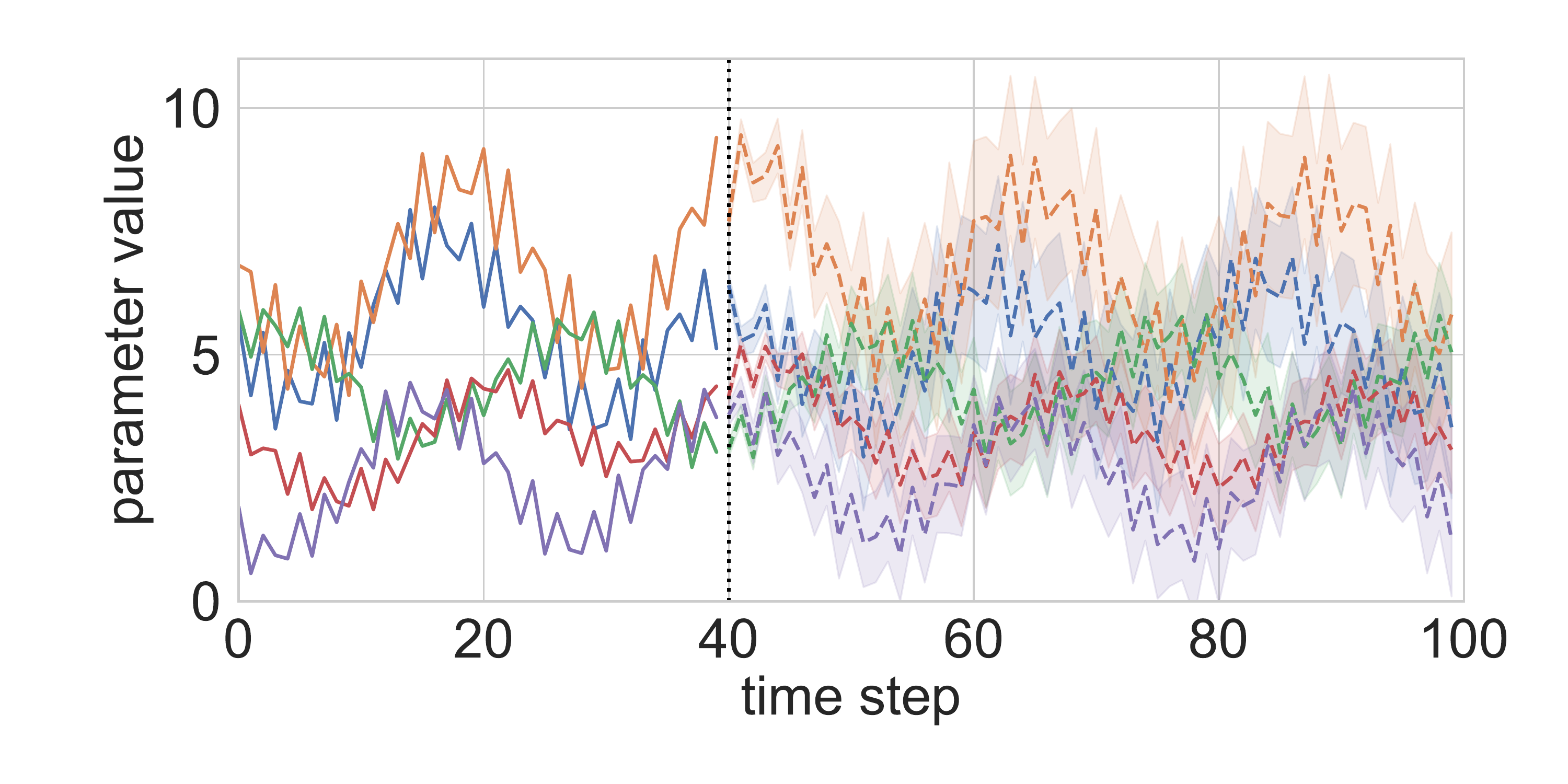}
        \caption{The learner predicts the future parameters with uncertainty. Each entry of the parameter corresponds to a time series prediction problem.}
        \label{fig:prediciton-with-uncertainty}
    \end{subfigure}
    \hfill
    \begin{subfigure}{0.32\linewidth}
        \centering
        \includegraphics[width=\textwidth]{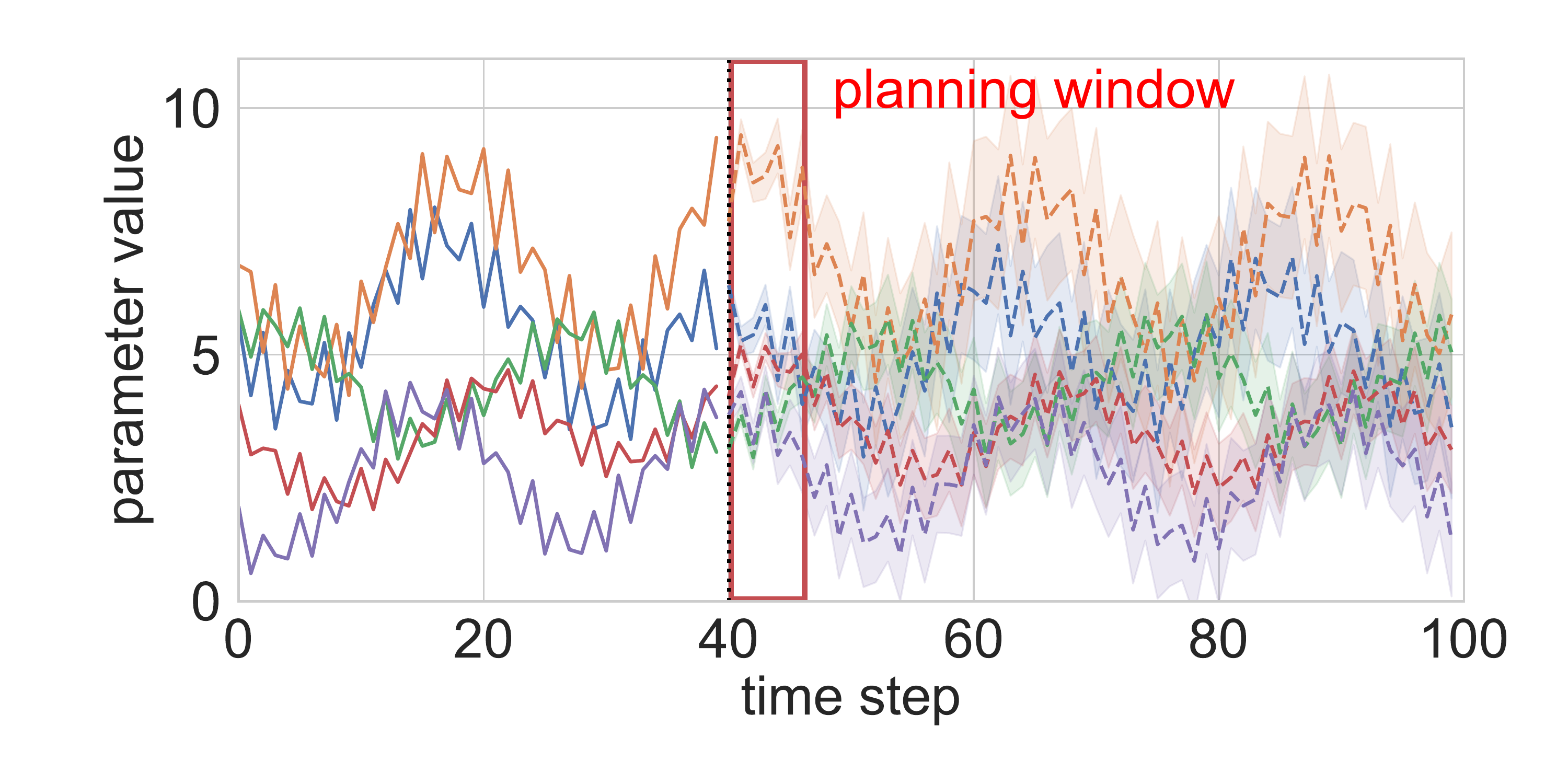}
        \caption{Given the predictions, we choose a dynamic planning window such that the total uncertainty within the window is of the same order of the switching cost.}
        \label{fig:planning-window}
    \end{subfigure}
    \hfill
    \begin{subfigure}{\linewidth}
        \centering
        \includegraphics[width=0.85\textwidth]{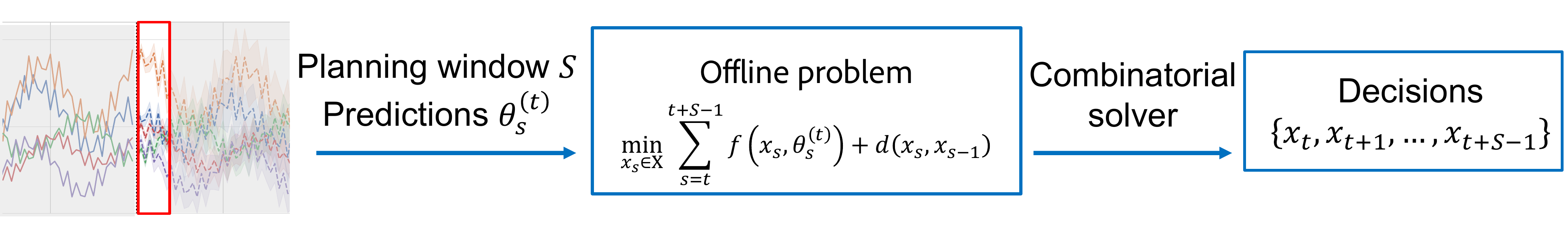}
        \caption{Given the predictions and the planning window, the planning problem reduces to an offline combinatorial problem. We can use any combinatorial solver to find a solution to the offline problem. The solution is executed in the planning window.}
        \label{fig:offline-problem}
    \end{subfigure}
    \caption{This flowchart summarizes how predictions are used to derive planning decisions. Fig.~\ref{fig:historical-data} shows the historical data prior to time $t$ as multiple time series. Fig.~\ref{fig:prediciton-with-uncertainty} visualizes the predictions and uncertainty intervals learned from the historical parameters. Fig.~\ref{fig:planning-window} demonstrates how to determine the dynamic planning window. Fig.~\ref{fig:offline-problem} solves an offline problem and executes accordingly.}
    \label{fig:flowchart}
\end{figure*}

\section{Planning Using Predictions}\label{sec:proactive-planning}
Motivated by the use of predictions in smoothed online convex optimization~\cite{chen2016using,li2020leveraging,antoniadis2020online}, this section studies the connection of predictions and predictive uncertainties to the dynamic regret.
To conduct the regret analysis below, we require the following assumptions to hold:
\begin{assumption}\label{assumption:lipschitzness}
The cost function $f(\decision, \parameter)$ is Lipschitz in $\parameter \in \parameterset$ with Lipschitz constant $L$, i.e.,  $\| \frac{\partial f (\decision, \parameter) }{\partial \parameter} \| \leq L$ for all $\decision \in \decisionset$ and $\parameter \in \parameterset$.
\end{assumption}

\begin{assumption}\label{assumption:bounded-switching-cost}
The switching cost is upper bounded in the feasible region $\decisionset$ by $B = \sup_{x, y \in \decisionset} d(\decision, \boldsymbol y)$.
\end{assumption}

Assumption~\ref{assumption:lipschitzness} quantifies the change of the cost function with respect to the parameter.
Assumption~\ref{assumption:bounded-switching-cost} quantifies the upper bound of switching cost.

\subsection{Predictions with Uncertainty}
\begin{assumption}\label{assumption:predictive-model}
We assume there is a predictive model that is trained based on the revealed parameters prior to time $t$. At time $t$, the predictive model takes the feature $\xi_t$ and produces a sequence of predicted future parameters $\{ \parameter^{(t)}_{s} \}_{s \in \N, s \geq t }$ with uncertainty $\{ \error^{(t)}_s \}_{s \in \N, s \geq t }$, where the distance between the prediction $\parameter^{(t)}_{s}$ and the true parameter $\parameter_{s}$ at time $s$ is bounded by $\| \parameter_{s} - \parameter^{(t)}_{s} \| \leq \error^{(t)}_s$.
\end{assumption}
We also assume that the predictive uncertainty $\error^{(t)}_s$ increases in $s$ due to the difficulty of predicting further future parameters, while the predictive uncertainty decreases in $t$ due to more training data available to train the predictive model.


\subsection{Planning in Fixed Time Horizon}
We first analyze the regret in fixed time horizon when we use the predictions to plan accordingly.
More precisely, at time $t$, given the previous decision $\decision_{t-1}$ at time $t-1$ and the prediction $\{ \parameter^{(t)}_{s} \}_{s \in \N, s \geq t }$ of the future time steps, the learner selects a planning window $S \in \N$ to plan for the next $S$ time steps by solving a minimization problem:
\begin{align}
     & \{ \decision_{s} \}_{s \in \{ t, t+1, \cdots, t+S-1 \}} \nonumber \\
    &= \argmin_{ \decision_{s} \in \decisionset ~\forall s} ~ \sum\nolimits_{s = t}^{t+S-1} f(\decision_{s}, \parameter^{(t)}_{s}) + d(\decision_{s}, \decision_{s-1}). \label{eqn:finite-optimization-problem}
\end{align}

Solving the above finite time horizon optimization problem suggests a solution $\{ \decision_{s} \}_{s \in \{t,t+1,\cdots, t+S-1 \}}$ in the next $S$ time steps to execute starting from time $t$.
This process is summarized in Fig.~\ref{fig:flowchart}.

However, since the predictions are not perfect, the suggested solution might not be the true optimal solution when the true cost function parameters are present. To compare with the true offline optimal solution using the true cost function parameters, we express the offline solution by:
\begin{align}
     & \{ \decision'_{s} \}_{s \in \{t,t+1,\cdots, t+S-1 \}} \nonumber \\
    & = \argmin\limits_{ \decision_{s} \in \decisionset_{s} ~\forall s} ~ \sum\nolimits_{s=t}^{t+S-1} f(\decision_{s}, \parameter_{s}) + d(\decision_{s}, \decision_{s-1}) \label{eqn:finite-optimization-problem-optimal}
\end{align}

The only difference between Eq.~\eqref{eqn:finite-optimization-problem} and Eq.~\eqref{eqn:finite-optimization-problem-optimal} is that Eq.~\eqref{eqn:finite-optimization-problem-optimal} has full access to the future cost parameters, while Eq.~\eqref{eqn:finite-optimization-problem} uses the predictions instead. We can define the difference by the following regret:
\begin{align}
     & \text{Reg}_{t}^{t+S-1}(\decision_{t-1}) \!
     = \! \left( \sum\nolimits_{s=t}^{t+S-1} f(\decision_{s}, \parameter_{s}) \! + \! d(\decision_{s}, \decision_{s-1}) \right) \! \nonumber \\
     & - \! \left( \sum\nolimits_{s=t}^{t+S-1} f(\decision'_{s}, \parameter_{s}) \! + \! d(\decision'_{s}, \decision'_{s-1}) \right). \label{eqn:partial-regret-definition}
\end{align}

We have the following bound on the regret:
\begin{restatable}{theorem}{segmentedRegret}\label{thm:finite-proactive-planning}
Under Assumption~\ref{assumption:lipschitzness}, the regret from time step $t$ to $t+S-1$ in Eq.~\ref{eqn:partial-regret-definition} is upper bounded by:
$
    \mathrm{Reg}_{t}^{t+S-1}(\decision_{t-1}) \leq 2 L \sum _{s = t}^{t+S-1}  \error^{(t)}_{s}.
$
where $L$ is the Lipschitz constant in Assumption~\ref{assumption:lipschitzness}.
\end{restatable}

Theorem~\ref{thm:finite-proactive-planning} links the dynamic regret with the total predictive uncertainty in finite time horizon.
Notice that the switching cost terms in Eq.~\eqref{eqn:partial-regret-definition} are misaligned. Therefore, the proof requires not only the Lipschitzness of the objective function $f$ but also the optimality conditions of both the offline and online planning problems to bound the total cumulative regret.

\begin{algorithm}[t]
\caption{Dynamic Future Planning}
\label{alg:infinite-proactive-planning}
{\bf Input:} Total time steps $T$. Maximal switching cost $B$. A predictive model that can produce predictions $\{ \parameter^{(t)}_{t+s} \}_{s \in \N}$ at time $t$. \\
\begin{algorithmic}[1]
\STATE {\bf Initialization} $t=1$, \# of planning windows $I = 0$
\WHILE{$t \leq T$}
    \STATE Get predictions $\{ \parameter^{(t)}_{s} \}_{s \in \N, s \geq t}$ and predictive uncertainty $\{ \error^{(t)}_{s} \}_{s \in \N, s \geq t}$ from the model.
    \STATE Find the largest $S$ s.t. $2 L \sum\nolimits _{s=t}^{t+S-1} \error^{(t)}_{s} \leq B$.
    \STATE Solve the optimization problem in Eq.~\eqref{eqn:finite-optimization-problem} with starting time $t$ and planning window $S$ to get $\{ \decision_{s} \}_{s \in \{t,t+1,\cdots,t+S-1 \} }$. \label{line:optimization-step}
    \STATE Execute $\decision_{s}$ and receive $\parameter_s$ with cost $f(\decision_s, \parameter_s) + d(\decision_s, \decision_{s-1})$ at time $s \in \{t,\cdots, t+S-1 \}$.
    \STATE Set $t = t + S$, $I = I + 1$.
\ENDWHILE 
\end{algorithmic}
\end{algorithm}

\subsection{Infinite Time Horizon and Dynamic Planning Window}
In the inifinite time horizon problem, the main idea is to reduce the problem to multiple finite time horizon problems with different planning window sizes.

Recall that the predictive uncertainty often increases when we try to predict the parameters in the far future, i.e., $\error^{(t)}_s$ is increasing in $s$. 
Since the regret in Theorem~\ref{thm:finite-proactive-planning} directly relates to the predictive uncertainty in the planning window, it suggests keeping the planning window small to reduce the regret.

On the other hand, Theorem~\ref{thm:finite-proactive-planning} assumes an identical initial decision $\decision_{t-1}$ in the online problem (Eq.~\eqref{eqn:finite-optimization-problem}) and offline problem (Eq.~\eqref{eqn:finite-optimization-problem-optimal}).
In the infinite time horizon case, two algorithms may start from different initial decisions, which may create an additional regret upper bounded by the maximum switching cost $B$ due to the misalignment of the initial decision. This observation suggests using larger planning windows to avoid changing between different planning windows.

Therefore, we propose to balance two sources of regret by choosing the largest planning window $S$ such that:
\begin{align}\label{eqn:planning-window-choice-test}
    2 L \sum\nolimits_{s=t}^{t+S-1} \error^{(t)}_{s} \leq B
\end{align}
The choice of the dynamic planning window can ensure that the total excessive predictive uncertainty is upper bounded by cost $B$, while we also plan as far as possible to reduce the number of planning windows incurred during switching between different finite time horizons.
The algorithm is described in Algorithm~\ref{alg:infinite-proactive-planning}.

\begin{restatable}{theorem}{totalRegret}\label{thm:dynamic-proactive-planning}
Given Lipschitzness $L$ in Assumption~\ref{assumption:lipschitzness} and the maximal switching cost $B$ in Assumption~\ref{assumption:bounded-switching-cost}, in $T$ time steps, Algorithm~\ref{alg:infinite-proactive-planning} achieves cumulative regret upper bounded by $2BI$, where $I$ is the total number of planning windows used in Algorithm~\ref{alg:infinite-proactive-planning}.
\end{restatable}
\begin{proof}[Proof sketch]
The regret of our algorithm comes from two parts: (i) regret from the discrepancy of the initial decision $\decision_{t-1}$ and the initial decision of the offline optimal $\decision^*_{t-1}$ at time $t$, the start of every planning window, and (ii) the incorrect predictions used in the optimization, which is bounded by Theorem~\ref{thm:finite-proactive-planning}.

The regret in part (i) is bounded by $d(\decision_{t-1}, \decision^*_{t-1}) \leq B$ for every planning window because it would take at most the maximal switching cost $B$ to align different initial decisions before we can compare. Thus the total regret in part (i) is bounded by $BI$, where $I$ is the number of planning windows executed in Algorithm~\ref{alg:infinite-proactive-planning}.

The regret in part (ii) is bounded by Theorem~\ref{thm:finite-proactive-planning} and the choice of the dynamic planning window in Eq.~\eqref{eqn:planning-window-choice-test}. We have $\mathrm{Reg}_t^{t+S-1}(\decision^*_{t-1}) \leq 2L\sum _{s=t}^{t+S_i-1} \error_s^{(t)} \leq B$ for the $i$-th window. We can take summation over all planning windows to bound the total regret in part (ii) by:
$
\sum _{i=1}^I B = BI.
$
where combining two bounds concludes the proof.
\end{proof}

Theorem~\ref{thm:dynamic-proactive-planning} links the excessive dynamic regret to $I$, the number of planning windows that Algorithm~\ref{alg:infinite-proactive-planning} uses.
The next step is to bound the number of planning windows $I$ by the total time steps $T$.
In Theorem~\ref{thm:regret-bound-full}, we first show that the cumulative regret is always sublinear in $T$ when the predictive uncertainty converges to $0$ when more data is collected.



\begin{restatable}{theorem}{regretBoundFull}\label{thm:regret-bound-full}
Under Assumption~\ref{assumption:lipschitzness} and~\ref{assumption:bounded-switching-cost}, if $\error^{(t)}_{t+s-1} = o(1)$ in $t$ for all $s \in \N$, i.e., $\error^{(t)}_{t+s-1} \rightarrow 0$ when $t \rightarrow \infty$, then the cumulative regret of Algorithm~\ref{alg:infinite-proactive-planning} is sublinear in $T$.
\end{restatable}
\begin{proof}
When the predictive uncertainty $\error_s^{(t)} \rightarrow 0$ when $t \rightarrow \infty$, the window size $S_t$ that satisfies $2L\sum _{s=t}^{t+S_i-1} \error_s^{(t)} \leq B$ at time $t$ converges to $\infty$ when $t \rightarrow \infty$. This suggests that the number of windows $I$ required in total number of time steps $T$ is strictly smaller than $\Theta(T)$, i.e., $I = o(T)$. By Theorem~\ref{thm:dynamic-proactive-planning}, the cumulative regret is upper bounded by $2BI = o(T)$, which is sublinear in $T$.
\end{proof}

Theorem~\ref{thm:regret-bound-full} guarantees that the cumulative regret of Algorithm~\ref{alg:infinite-proactive-planning} in Theorem~\ref{thm:dynamic-proactive-planning} is sublinear when the uncertainty converges to $0$. This establishes the no-regretness of Algorithm~\ref{alg:infinite-proactive-planning} in dynamic regret, which is only known to be possible in the smoothed online {\it convex} optimization but not known in the smoothed online {\it combinatorial} optimization.

In some special cases of the predictive uncertainty, we can further provide a more precise bound on the cumulative regret in the following corollary.
\begin{restatable}{corollary}{regretBoundPoly}\label{coro:regret-bound}
If the uncertainty satisfies $\error^{(t)}_{t+s-1} = O(\frac{s^a}{t^b})$, $\forall s, t \in \N$ with $a, b \in \R_{\geq 0}$, we have:
\begin{align*}
    \mathrm{Reg}_T \leq \begin{cases}
    O(T^{1 - \frac{b}{a+1}}) & \text{if $b < a+1$} \\
    O(\log T) & \text{if $b = a+1$} \\
    O(\log \log T) & \text{if $b > a+1$}
    \end{cases}.
\end{align*}
\end{restatable}

Corollary~\ref{coro:regret-bound} is proved by providing a more concrete bound on the number of planning windows $I$ in Theorem~\ref{thm:dynamic-proactive-planning}.
Corollary~\ref{coro:regret-bound} also quantifies the dependency of the cumulative regret on the convergence rate of predictive uncertainty. When $b >0$, the cumulative regret is always sublinear, which matches our result in Theorem~\ref{thm:regret-bound-full}.


\subsection{Lower Bound on The Cumulative Regret}
In this section, we provide a lower bound on the expected cumulative regret, showing that no randomized algorithm can achieve an expected cumulative regret lower than a term similar to the upper bound.

\begin{restatable}{corollary}{lowerBound}\label{coro:lower-bound}
Given $\epsilon^{(t)}_{t+s-1} = \Omega(\frac{s^a}{t^b})$ for all $t,s \in \N$ with $0 \leq b$, there exist instances such that for any randomized algorithm, the expected regret is at least: 
\begin{align}\nonumber
    \E[\mathrm{Reg}_T] \geq \begin{cases}
    \Omega(T^{1-b}) & \text{if $b < 1$} \\
    \Omega(\log T) & \text{if $b = 1$} \\
    \Omega(1) & \text{if $b > 1$}
\end{cases}.
\end{align}
\end{restatable}

The lower bound suggests that there is no online learning algorithm that can achieve a cumulative regret that is smaller than the regret in Corollary~\ref{coro:lower-bound}.
Specifically, we can see that the lower bound matches to the upper bound up to a logarithm factor when $a=0$, which guarantees the tightness of our upper bound in Corollary~\ref{coro:regret-bound} and Theorem~\ref{thm:dynamic-proactive-planning} in the case of $a=0$.


\subsection{Extension to Probabilistic Bounds}
In this paper, we primarily focus on the deterministic uncertainty bounds of the predictive model.
The same analyses in Section~\ref{sec:proactive-planning} also generalize to probabilistic bounds of the predictive model that hold with high probability, e.g., with probability $1-\delta_i$ for each prediction in the $i$-th planning window with size $S_i$.
This kind of probabilistic bounds is commonly seen in the literature of probably approximately correct (PAC) learning, where the predictive error bound can be bounded by the number of training samples used in fitting the underlying hypothesis class.
In this case, the regret analysis in Theorem~\ref{thm:finite-proactive-planning} needs to additionally consider the event when the uncertainty bounds do not hold, which leads to an additional regret term with order $O(S_i \delta_i)$ in Theorem~\ref{thm:finite-proactive-planning}, leading to a linear term $\sum\nolimits_{i=1}^I S_i \delta_i$ in Theorem~\ref{thm:dynamic-proactive-planning}.

Fortunately, we can also select a decreasing failure probability $\delta_i$ in the later planning windows when more samples are collected.
As long as we can guarantee that the choice of uncertainty bound $\error^{(t)}_s$ and the failure probability $\delta_i$ at time $t$ converge to $0$ when more samples are collected, we can obtain a similar result as Theorem~\ref{thm:regret-bound-full} showing the cumulative regret bound is sublinear in $T$. This generalizes our results of deterministic bounds to probabilistic bounds.



\section{Iterative Algorithm for Offline Problem with Switching Cost}\label{sec:iterative-algorithm}
Given imperfect predictions and the planning windows, we can reduce the online problem to an offline problem.
This section focuses on solving the following offline combinatorial optimization problem with switching cost.
\begin{align}\label{eqn:finite-optimization-problem-polished}
    \min_{ \decision_{t} \in \decisionset} ~ \sum\nolimits_{t = 1}^{S} f(\decision_{t}, \parameter_{t}) + d(\decision_{t}, \decision_{t-1}).
\end{align}

Solving Eq.~\eqref{eqn:finite-optimization-problem-polished} is challenging because the combinatorial structure of the decision $\decision_t \in \decisionset_t$ and the additional temporal dependency caused by the switching cost $d(\decision_{t}, \decision_{t-1})$.

\begin{algorithm}[tb]
\caption{Iterative algorithm for offline problems}
\label{alg:iterative-algorithm}
\begin{algorithmic}[1]
\STATE Let $J=10$ and $\decision_t = \decision_0$ for all $t \in [S]$. 
\FOR{$j \in [J]$}
    \FOR{$t \in [S]$}
        \STATE Let $c = 0.5$ if $j < J$ otherwise $c=1$.
        \STATE Solve Eq.~\eqref{eqn:iterative-operation} with $\decision_{t-1}$, $\decision_{t+1}, c$ to update $\decision_t$.
    \ENDFOR
\ENDFOR
\end{algorithmic}
\end{algorithm}



\subsection*{Decomposition and Iterative Algorithm}
If we fix the assignments $\decision_{t-1}, \decision_{t+1}$, finding the optimal solution at time step $t$ reduces to the following problem with $c=1$:
\begin{align}\label{eqn:iterative-operation}
    \decision_t = \argmin_{\decision \in \decisionset_t} f(\decision, \parameter_t) + c(d(\decision, \decision_{t-1}) + d(\decision, \decision_{t+1})).
\end{align}

Compared to Eq.~\eqref{eqn:finite-optimization-problem-polished}, Eq.~\eqref{eqn:iterative-operation} avoids the temporal dependency across multiple time steps and largely reduces the number of binary variables.
In practice, solving Eq.~\eqref{eqn:iterative-operation} is more tractable than solving Eq.~\eqref{eqn:finite-optimization-problem-polished}.

This observation motivates the idea of iteratively fixing the neighbor decisions $\decision_{t-1}, \decision_{t+1}$ and updating the decision at time step $t$ for all $t \in [S]$.
We use $\decision_t = \decision_0$ to initialize all decisions. Then we can iteratively solve Eq.~\eqref{eqn:iterative-operation} with different $t$ to update the decision $\decision_t$. This method decouples the temporal dependency and reduces the problem to a standard combinatorial optimization of function $f$ with additional regularization terms. We can use mixed integer linear program or any other approximation algorithms to solve Eq.~\eqref{eqn:iterative-operation}.

Moreover, we can notice that any improvement made by solving Eq.~\eqref{eqn:iterative-operation} with $c=1$ provides the same improvement to Eq.~\eqref{eqn:finite-optimization-problem-polished}. This suggests that the optimal decision of Eq.~\eqref{eqn:finite-optimization-problem-polished} is a fixed point of Eq.~\eqref{eqn:iterative-operation} when $c=1$.

\begin{restatable}{theorem}{fixedPoint}\label{thm:fixed-point}
The optimal sequence $\{ \decision^*_t \}_{t \in [S]}$ of Eq.~\eqref{eqn:finite-optimization-problem-polished} is a fixed point of Eq.~\eqref{eqn:iterative-operation} with $c=1$.
\end{restatable}

\begin{proof}
Suppose that $\{ \decision_s^* \}_{s \in [S]} $ optimizes Eq.~\eqref{eqn:finite-optimization-problem-polished}. 
For any $t \in [S]$, if we can find $\decision'_t$ gets a positive improvement in Eq.~\eqref{eqn:iterative-operation} with $c=1$:
\begin{align}
    0 < \delta = & \left( f(\decision^*_t, \parameter_t) + d(\decision^*_t, \decision^*_{t-1} + d(\decision^*_t, \decision^*_{t+1}) \right) - \left( f(\decision'_t, \parameter_t) + d(\decision'_t, \decision^*_{t-1}) + d(\decision'_t, \decision^*_{t+1}) \right) \nonumber
\end{align}

Then the new sequence $\{\decision_1^*, \cdots, \decision_{t-1}^*, \decision'_{t}, \decision_{t+1}^*, \cdots, \decision_{S}^* \}$ gets the same improvement with:
\begin{align}
    & \text{cost}(\{\decision_s^*\}_{s \in [S]}) - \text{cost}(\{\decision_1^*, \cdots, \decision_{t-1}^*, \decision'_{t}, \decision_{t+1}^*, \cdots, \decision_{S}^* \}) = \delta > 0
\end{align}
where $\text{cost}(\{\decision_1^*, \cdots, \decision_{t-1}^*, \decision'_{t}, \decision_{t+1}^*, \cdots, \decision_{S}^* \})$ is strictly smaller than the optimal value $\text{cost}(\{\decision_s^*\}_{s \in [S]})$, which violates the optimality assumption of $\{ \decision^*_s \}_{s \in [S]}$. This implies that we cannot find $\decision'_t$ that gives a strictly smaller objective in Eq.~\eqref{eqn:iterative-operation}, which also implies that $\decision^*_{t}$ is a fixed point to Eq.~\eqref{eqn:iterative-operation} with $c=1$ using $\decision^*_{t-1}$ and $\decision^*_{t+1}$ as the neighbor decisions.
\end{proof}

Theorem~\ref{thm:fixed-point} ensures that the iterative process in Eq.~\ref{eqn:iterative-operation} stops updating at the optimal solution.
However, in practice, there could be multiple fixed points and suboptimal points due to the combinatorial structure.
This can be problematic because the iterative process in Eq.~\ref{eqn:iterative-operation} can stop at many different suboptimal solution without further improving the solution quality.
To avoid getting stuck by suboptimal solutions, we use a smaller scaling constant $c=0.5$ to relax the iterative update, and use $c=1$ in the final step to strengthen the solution. The iterative algorithm is described in Algorithm~\ref{alg:iterative-algorithm}, which can be used to replace Line~\ref{line:optimization-step} in Algorithm~\ref{alg:infinite-proactive-planning}.

\section{Experiment Setup}
In our experiment, we use the distributed streaming system problems with synthetic data to compare our algorithm with other baselines.

\paragraph{Cost function and switching cost}
In the distributed streaming system, the learner maintains a bipartite assignment $\boldsymbol\decision_t \in \decisionset_t \subseteq \{0,1\}^{k \times m}$ between $k$ topics and $m$ servers at time step $t$ to process the streaming data, where $\boldsymbol\decision_{t,i,j} = 1$ denotes that topic $i$ is assigned to server $j$ at time $t$ to process the incoming traffic.
Once the decision $\boldsymbol\decision_t$ is chosen at time $t$, a traffic vector $\parameter_t \in \parameterset \subseteq \R^k$ is revealed.

Given traffic $\parameter_t$ and the chosen assignment $\decision_t$, we define the cost function by $f(\decision_t, \parameter_t) = \| \decision_t^\top \parameter_t \|_\infty$ as the resulting server imbalance cost, which is also known as makespan, i.e., the maximal number of messages a server needs to process across all servers. Minimizing makespan is a well-studied strongly NP-complete problem~\cite{garey1979computers} with various approximation algorithms~\cite{hochbaum1987using,leung1989bin}. Additionally, we define the switching cost by $d(\decision, \boldsymbol y) \coloneqq \boldsymbol1_k^\top |\decision - \boldsymbol y| \boldsymbol u $, where $|\decision - \boldsymbol y| \in \R^{k \times m}_{\geq 0}$ represents the number of switches of each pair of topic and server, and each entry of $\boldsymbol u \in \R^m$ denotes the unit switching cost associated to the corresponding server, which is randomly drawn from a uniform distribution $U[0,2]$.


\paragraph{Data generation}
We assume that there are $k=10$ topics to be assigned to $m=3$ servers.
We generate $k$ time series, where each represents the trend of incoming traffic $\{ \parameter_{t,i} \}_{t \in [T]}$ of topic $i \in [k]$ as the cost function parameter. Each time series is generated by a composition of sine waves, an autoregressive process, and a Gaussian process to model the seasonality, trend, and the random process.
We use sine waves with periods of $24$ and $2$ with amplitudes drawn from $U[1,2]$ and $U[0.5,1]$ to model the daily and hourly changes. We use an autoregressive process AR(1) that takes the weighted sum of $0.9$ of the previous signal and a $0.1$ of a white noise to generate the next signal. Lastly, we use a rational quadratic kernel as the Gaussian process kernel.

\paragraph{Predictive model}
At time step $t$, to predict the incoming traffic $\parameter_s \in \parameterset \subseteq \R^k$ for all $s \geq t$, we collect all the historical data $\{  \parameter_{s'} \}_{s' < t}$ prior to time $t$ and apply Gaussian process regression using the same rational quadratic kernel on the historical data to generate predictions $\{ \parameter^{(t)}_{s} \}_{s \geq t}$ of the future time steps. We use the standard deviation learned from Gaussian process regression as the uncertainty $\{ \error^{(t)}_s \}_{s \geq t}$.

\paragraph{Experimental setup}
For each instance of the load balancing problem, we assume $50$ historical data have been collected a priori to stabilize Gaussian process regression. We run different online algorithms for another $100$ time steps with hidden incoming traffic to measure the performance of online algorithms. For each setup, we run $10$ independent trials with different random seeds to estimate the average performance. All the results are plotted with average value and the corresponding standard deviation.

\begin{figure*}[t]
    \centering
    \begin{subfigure}{0.325\linewidth}
        \centering
        \includegraphics[width=\textwidth]{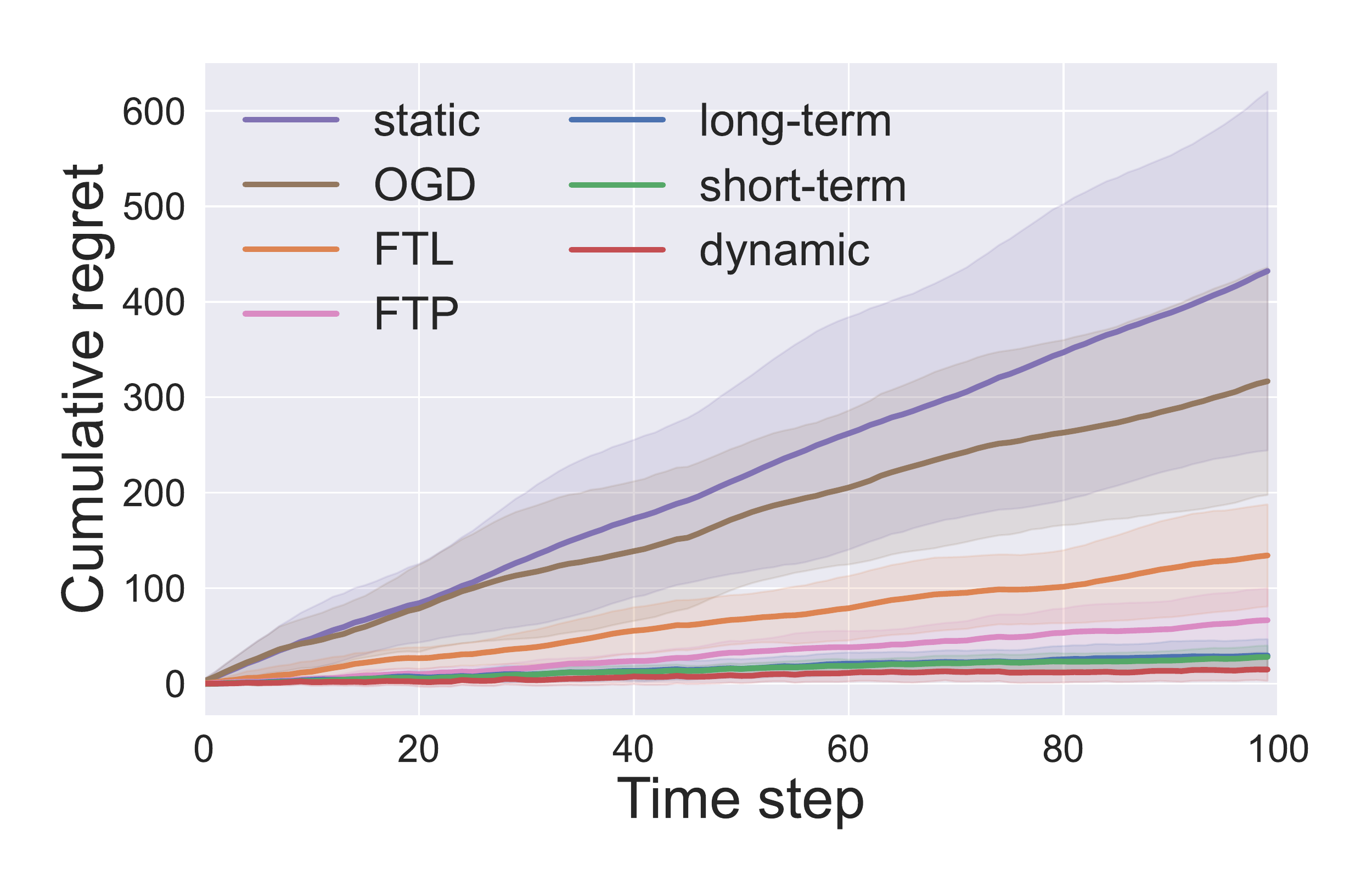}
        \caption{Average cumulative regret that includes the imbalance cost and switching cost.}
        \label{fig:cumulative-regret}
    \end{subfigure}
    \hfill
    \begin{subfigure}{0.325\linewidth}
        \centering
        \includegraphics[width=\textwidth]{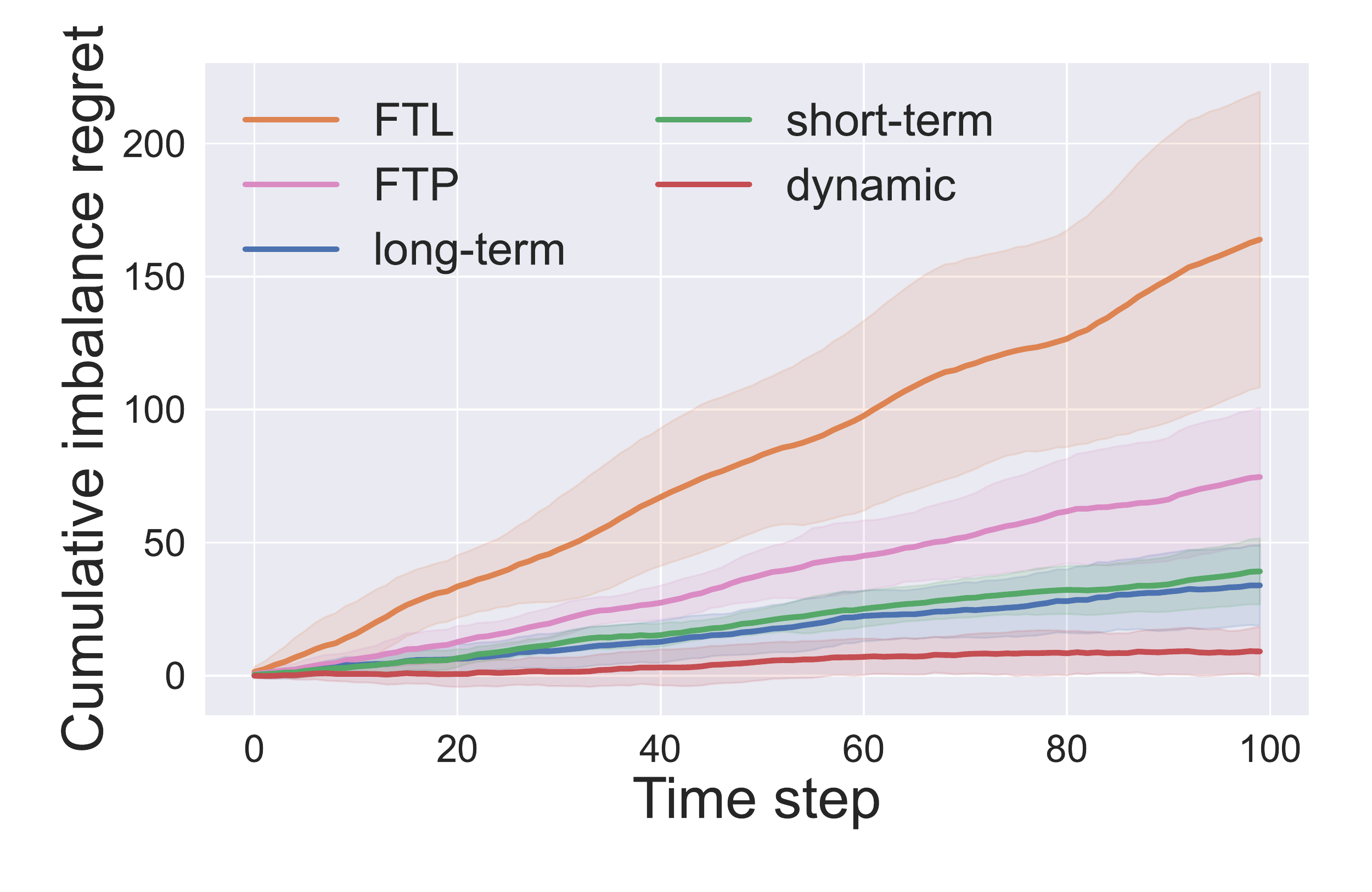}
        \caption{Average cumulative imbalance regret compared to the offline benchmark.}
        \label{fig:cumulative-imbalance-cost}
    \end{subfigure}
    \hfill
    \begin{subfigure}{0.325\linewidth}
        \centering
        \includegraphics[width=\textwidth]{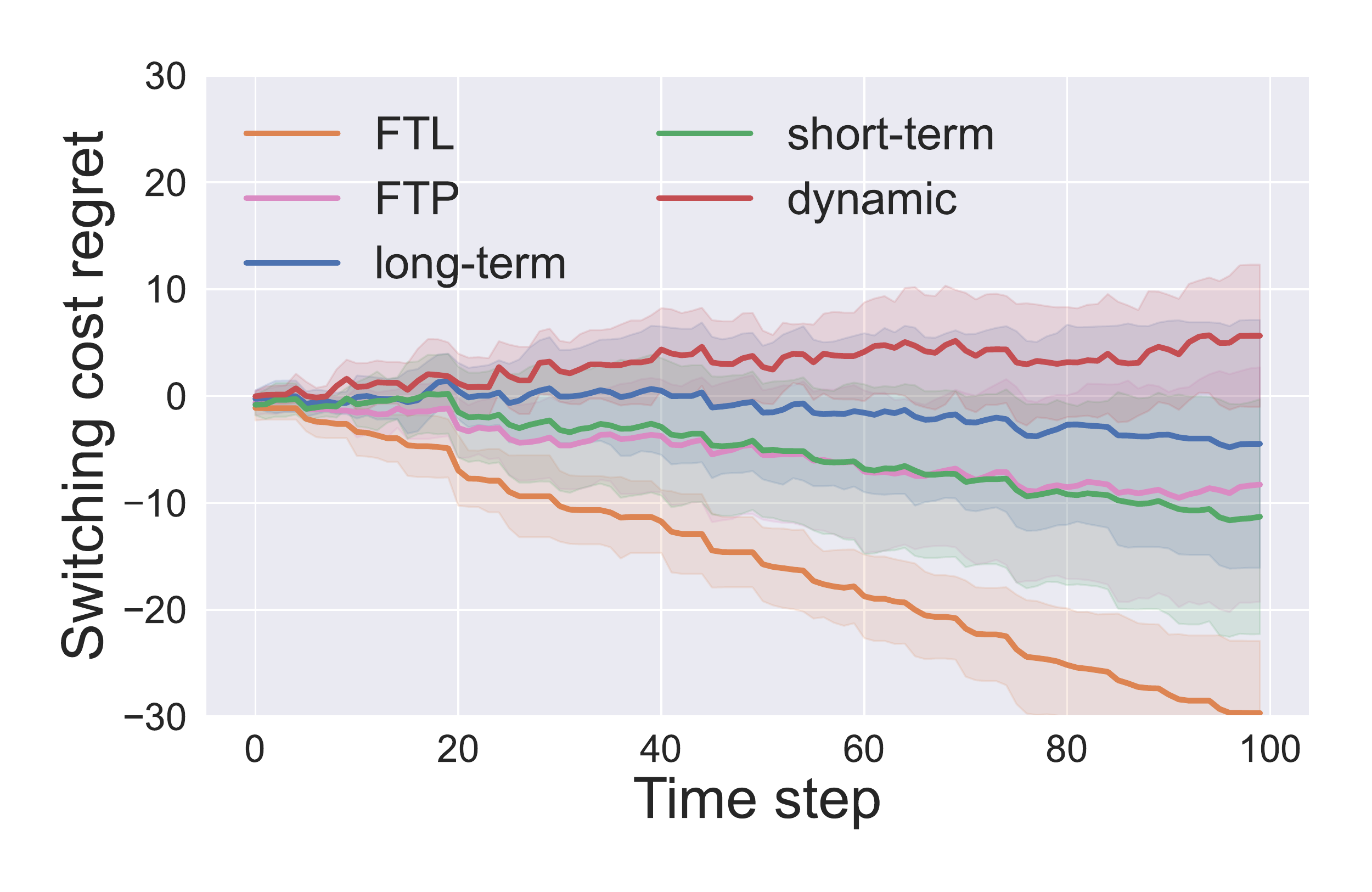}
        \caption{Average cumulative switching cost regret compared to the offline benchmark.}
        \label{fig:cumulative-switching-cst}
    \end{subfigure}
    \caption{We compare the performance of our approaches with various baselines without using predictions. The first takeaway is that methods using predictions largely outperform the methods without using predictions in Fig.~\ref{fig:cumulative-regret}. Secondly, choosing the right planning window can achieve a better imbalance cost in  Fig.~\ref{fig:cumulative-imbalance-cost} with a small increase in the amount of switching cost in  Fig.~\ref{fig:cumulative-switching-cst}. All the algorithms are compared with an offline benchmark with full information. The shaded area refers to the region within first standard deviation.
    }
    \label{fig:cost-experiment}
\end{figure*}

\begin{figure}[t]
    \centering
    \begin{subfigure}{0.48\linewidth}
        \centering
        \includegraphics[width=\textwidth]{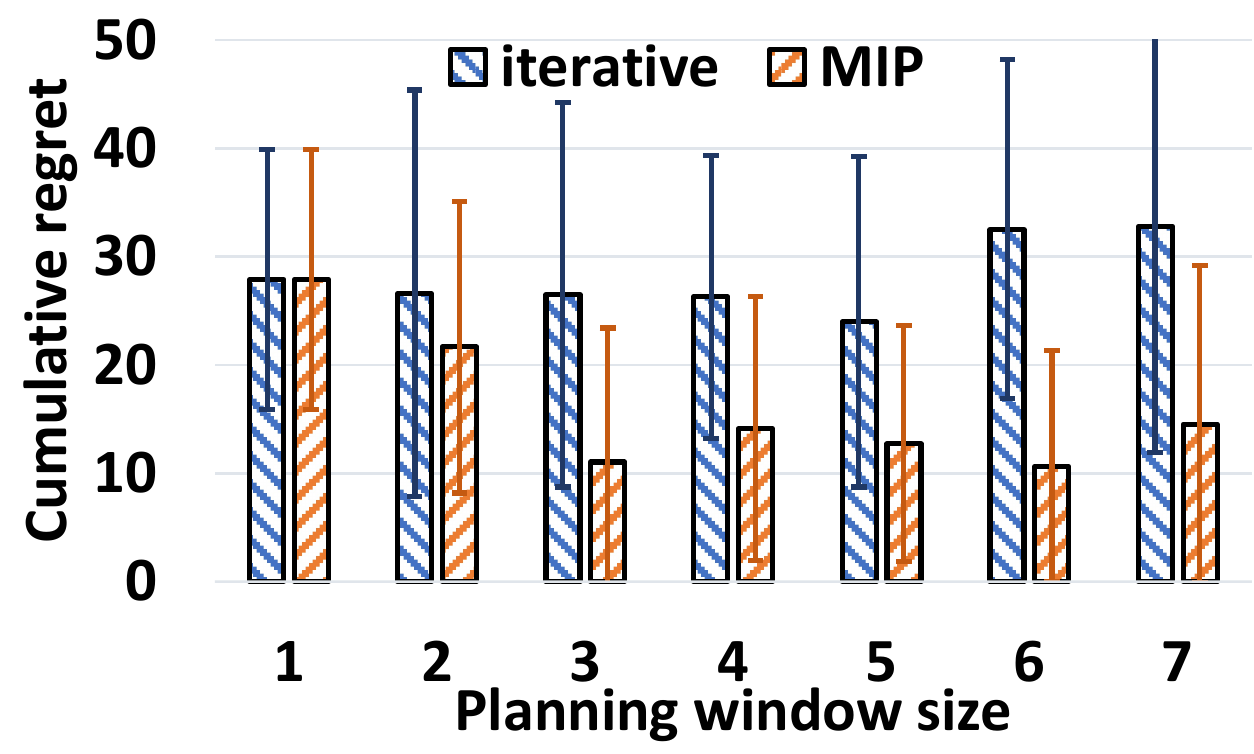}
        \caption{Cumulative regret of methods using different planning window sizes and different optimization approaches.}
        \label{fig:changing-window-size}
    \end{subfigure}
    \hspace{1mm}
    \begin{subfigure}{0.48\linewidth}
        \centering
        \includegraphics[width=\textwidth]{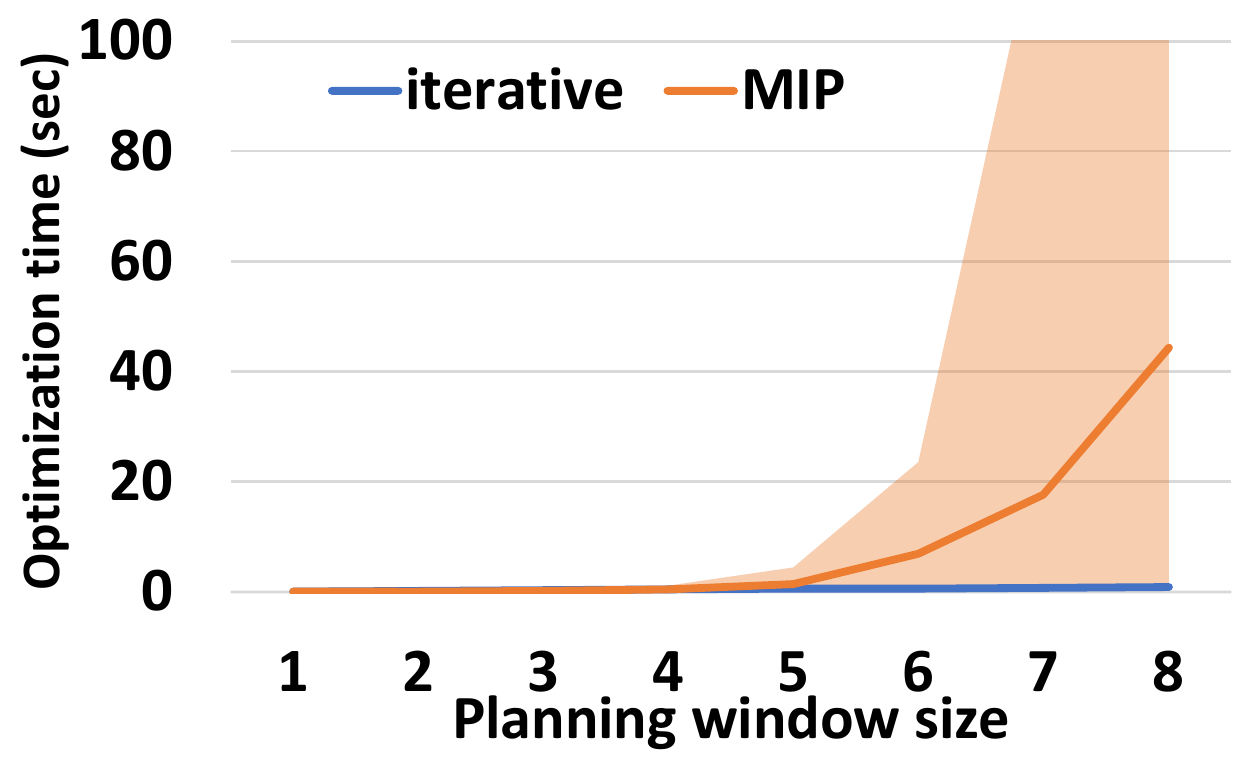}
        \caption{Average running time per optimization of different optimization methods with different planning window sizes.}
        \label{fig:runtime}
    \end{subfigure}
    \caption{Comparison of different methods of solving Eq.~\eqref{eqn:finite-optimization-problem} and different planning window sizes.}
    \label{fig:parameter-experiment}
\end{figure}
\section{Experimental Results}
We compare with our algorithm with baselines in the literature of online convex optimization:
\begin{itemize}
    \item The {\bf static} approach uses the initial assignment and never adjusts dynamically.
    \item The Online Gradient Descent ({\bf OGD}) updates the assignment by running gradient descent on the cost function received previously and project back to the discrete feasible region.
    \item The Follow-The-Leader ({\bf FTL}) aggregates all the cost functions received in the past and finds the optimal decision that optimizes the historical cost functions with switching cost.
    \item The Follow-The-Previous ({\bf FTP}) optimizes the cost function in the last time step.
    \item The {\bf short-term} algorithm and the {\bf long-term} algorithm both use predictions but with deterministic planning window sizes $1$ and $10$, respectively.
    \item The {\bf dynamic} algorithm refers to our algorithm using a dynamic planning window determined by the predictive uncertainty.
\end{itemize}

All the algorithms compare with an offline benchmark with full information. Since the offline problem is NP-hard to solve, we split the offline problem into chunks of size $5$ and solve each of them optimally using mixed integer program to get the offline performance.

\paragraph{Effect of predictions}
In Fig.~\ref{fig:cost-experiment}, we compare the performance of baselines (static, OGD, FTL, FTP) with approaches using predictions with different planning window sizes (short-term, long-term, dynamic).
We first notice that OGD and FTL perform worse than FTP, which simply follows the previous cost function to update solution. Due to the smoothness of the cost function parameters, optimizing over the previous cost function can be a strong baseline.

Secondly, the methods using predictions further improve the solution quality. Using predictions can help leverage the seasonality and trend information, and leave the uncertainty to the planning part. On the other hand, the OGD and the FTL algorithms are designed to deal with the case without predictable pattern and switching cost. The different purposes of algorithm design make our algorithm more applicable to our problem.

Lastly, in Fig.~\ref{fig:cumulative-regret}, we can see that the dynamic algorithm achieves the smallest cumulative regret compared to the short-term algorithm and the long-term algorithm using planning window with size $1$ and $10$, respectively. 
Fig.~\ref{fig:cumulative-imbalance-cost} and Fig.~\ref{fig:cumulative-switching-cst} further compare different performance metrics. We can see that our approach of choosing proper planning window can achieve much smaller server imbalance performance while requiring slightly more switching cost only.
Methods considering less future effect (FTL, FTP, short-term) can be reluctant to switch and underestimate the benefit of switching, which results in a smaller switching cost but larger imbalance cost.
In contrast, the long-term algorithm using larger planning window instead can be harmed by the increasing predictive uncertainty, which leads to incorrect planning decision due to the uncertainty. This result justifies the benefit of predictions and the right planning window to balance between uncertainty and the switching cost.

\paragraph{Effect of planning window size}
In Fig.~\ref{fig:changing-window-size}, we compare the performance of different choices of planning window size and different ways of solving the offline problem in Eq.~\eqref{eqn:finite-optimization-problem}.
First, if we use mixed integer program ({\bf MIP}), we can see a clear improvement by using a larger planning window and a slightly degraded performance after window size exceeds $3$. This empirical result matches to our analysis of shorter and longer planning windows, where the dynamic planning window suggests a planning window with size around $3$.
We also compare with an {\bf iterative} algorithm (Algorithm~\ref{alg:iterative-algorithm} in Appendix~\ref{sec:iterative-algorithm}) that is used to approximately solve the NP-hard offline problem in Eq.~\eqref{eqn:finite-optimization-problem}. The effect of planning window size is less significant due to the suboptimality of the iterative algorithm. But we can still see a similar benefit while using an appropriate planning window size.

Fig.~\ref{fig:runtime} compares the runtime of solving Eq.~\eqref{eqn:finite-optimization-problem} using different approaches and planning window sizes. Runtime of solving the optimization problem is important because decisions have to be made in real time.
We can see that MIP requires an exponentially increasing runtime because the combinatorial structure and the linearly increasing number of binary variables when the window size grows. On the other hand, the iterative algorithm solves the problem approximately and more efficiently. In short, the MIP algorithm achieves the best performance but with an expensive computation, while the iterative algorithm scales better but with a loss in the solution quality.

\section{Conclusion}
This paper studies the smoothed online combinatorial optimization problem with switching cost.
We show that when predictions with uncertainty are available, we can bound the dynamic regret by the convergence of the predictive uncertainty, which links the bound on dynamic regret to the predictability of the incoming cost function parameters. Our analysis suggests using a dynamic planning window dependent on the sequence of predictive uncertainties. Our dynamic planning window can optimize the regret, where we empirically show in our experiments that using a predictive model and an appropriate planning window can further improve the performance.

\bibliography{reference}

\begin{thebibliography}{37}
\providecommand{\natexlab}[1]{#1}

\bibitem[{Andrew et~al.(2013)Andrew, Barman, Ligett, Lin, Meyerson, Roytman,
  and Wierman}]{andrew2013tale}
Andrew, L.; Barman, S.; Ligett, K.; Lin, M.; Meyerson, A.; Roytman, A.; and
  Wierman, A. 2013.
\newblock A tale of two metrics: Simultaneous bounds on competitiveness and
  regret.
\newblock In \emph{COLT}, 741--763.

\bibitem[{Antoniadis et~al.(2020)Antoniadis, Coester, Elias, Polak, and
  Simon}]{antoniadis2020online}
Antoniadis, A.; Coester, C.; Elias, M.; Polak, A.; and Simon, B. 2020.
\newblock Online metric algorithms with untrusted predictions.
\newblock In \emph{ICML}, 345--355.

\bibitem[{Audibert, Bubeck, and Lugosi(2014)}]{audibert2014regret}
Audibert, J.-Y.; Bubeck, S.; and Lugosi, G. 2014.
\newblock Regret in online combinatorial optimization.
\newblock \emph{Mathematics of Operations Research}, 39(1): 31--45.

\bibitem[{Badiei, Li, and Wierman(2015)}]{badiei2015online}
Badiei, M.; Li, N.; and Wierman, A. 2015.
\newblock Online convex optimization with ramp constraints.
\newblock In \emph{2015 54th IEEE Conference on Decision and Control (CDC)},
  6730--6736.

\bibitem[{Bartal, Bollob{\'a}s, and Mendel(2006)}]{bartal2006ramsey}
Bartal, Y.; Bollob{\'a}s, B.; and Mendel, M. 2006.
\newblock Ramsey-type theorems for metric spaces with applications to online
  problems.
\newblock \emph{Journal of Computer and System Sciences}, 72(5): 890--921.

\bibitem[{Bartal et~al.(2003)Bartal, Linial, Mendel, and
  Naor}]{bartal2003metric}
Bartal, Y.; Linial, N.; Mendel, M.; and Naor, A. 2003.
\newblock On metric Ramsey-type phenomena.
\newblock In \emph{STOC}, 463--472.

\bibitem[{Bartlett, Hazan, and Rakhlin(2007)}]{bartlett2007adaptive}
Bartlett, P.; Hazan, E.; and Rakhlin, A. 2007.
\newblock Adaptive online gradient descent.

\bibitem[{Bhat et~al.(2020)Bhat, Farias, Moallemi, and Sinha}]{bhat2020near}
Bhat, N.; Farias, V.~F.; Moallemi, C.~C.; and Sinha, D. 2020.
\newblock Near-optimal AB testing.
\newblock \emph{Management Science}, 66(10): 4477--4495.

\bibitem[{Bubeck et~al.(2021)Bubeck, Cohen, Lee, and Lee}]{bubeck2021metrical}
Bubeck, S.; Cohen, M.~B.; Lee, J.~R.; and Lee, Y.~T. 2021.
\newblock Metrical task systems on trees via mirror descent and unfair gluing.
\newblock \emph{SIAM Journal on Computing}, 50(3): 909--923.

\bibitem[{Bubeck et~al.(2020)Bubeck, Klartag, Lee, Li, and
  Sellke}]{bubeck2020chasing}
Bubeck, S.; Klartag, B.; Lee, Y.~T.; Li, Y.; and Sellke, M. 2020.
\newblock Chasing nested convex bodies nearly optimally.
\newblock In \emph{SODA}, 1496--1508. SIAM.

\bibitem[{Bubeck et~al.(2019)Bubeck, Lee, Li, and
  Sellke}]{bubeck2019competitively}
Bubeck, S.; Lee, Y.~T.; Li, Y.; and Sellke, M. 2019.
\newblock Competitively chasing convex bodies.
\newblock In \emph{STOC}, 861--868.

\bibitem[{Camacho and Alba(2013)}]{camacho2013model}
Camacho, E.~F.; and Alba, C.~B. 2013.
\newblock \emph{Model predictive control}.
\newblock Springer science \& business media.

\bibitem[{Chen et~al.(2015)Chen, Agarwal, Wierman, Barman, and
  Andrew}]{chen2015online}
Chen, N.; Agarwal, A.; Wierman, A.; Barman, S.; and Andrew, L.~L. 2015.
\newblock Online convex optimization using predictions.
\newblock In \emph{Proceedings of the 2015 ACM SIGMETRICS International
  Conference on Measurement and Modeling of Computer Systems}, 191--204.

\bibitem[{Chen et~al.(2016)Chen, Comden, Liu, Gandhi, and
  Wierman}]{chen2016using}
Chen, N.; Comden, J.; Liu, Z.; Gandhi, A.; and Wierman, A. 2016.
\newblock Using predictions in online optimization: Looking forward with an eye
  on the past.
\newblock \emph{ACM SIGMETRICS Performance Evaluation Review}, 44(1): 193--206.

\bibitem[{Farrell and Klemperer(2007)}]{farrell2007coordination}
Farrell, J.; and Klemperer, P. 2007.
\newblock Coordination and lock-in: Competition with switching costs and
  network effects.
\newblock \emph{Handbook of industrial organization}, 3: 1967--2072.

\bibitem[{Flaxman, Kalai, and McMahan(2004)}]{flaxman2004online}
Flaxman, A.~D.; Kalai, A.~T.; and McMahan, H.~B. 2004.
\newblock Online convex optimization in the bandit setting: gradient descent
  without a gradient.
\newblock \emph{arXiv preprint cs/0408007}.

\bibitem[{Friedman and Linial(1993)}]{friedman1993convex}
Friedman, J.; and Linial, N. 1993.
\newblock On convex body chasing.
\newblock \emph{Discrete \& Computational Geometry}, 9(3): 293--321.

\bibitem[{Garey and Johnson(1979)}]{garey1979computers}
Garey, M.~R.; and Johnson, D.~S. 1979.
\newblock \emph{Computers and intractability}, volume 174.
\newblock freeman San Francisco.

\bibitem[{Garg(2013)}]{garg2013apache}
Garg, N. 2013.
\newblock \emph{Apache kafka}.
\newblock Packt Publishing Birmingham.

\bibitem[{Gemp and Mahadevan(2016)}]{gemp2016online}
Gemp, I.; and Mahadevan, S. 2016.
\newblock Online Monotone Optimization.
\newblock \emph{arXiv preprint arXiv:1608.07888}.

\bibitem[{Hazan(2019)}]{hazan2019introduction}
Hazan, E. 2019.
\newblock Introduction to online convex optimization.
\newblock \emph{arXiv preprint arXiv:1909.05207}.

\bibitem[{Hazan, Agarwal, and Kale(2007)}]{hazan2007logarithmic}
Hazan, E.; Agarwal, A.; and Kale, S. 2007.
\newblock Logarithmic regret algorithms for online convex optimization.
\newblock \emph{Machine Learning}, 69(2-3): 169--192.

\bibitem[{Hochbaum and Shmoys(1987)}]{hochbaum1987using}
Hochbaum, D.~S.; and Shmoys, D.~B. 1987.
\newblock Using dual approximation algorithms for scheduling problems
  theoretical and practical results.
\newblock \emph{Journal of the ACM (JACM)}, 34(1): 144--162.

\bibitem[{Jia, Xu, and Liu(2017)}]{jia2017optimization}
Jia, Y.; Xu, W.; and Liu, X. 2017.
\newblock An optimization framework for online ride-sharing markets.
\newblock In \emph{2017 IEEE 37th International Conference on Distributed
  Computing Systems (ICDCS)}, 826--835. IEEE.

\bibitem[{Koolen et~al.(2010)Koolen, Warmuth, Kivinen
  et~al.}]{koolen2010hedging}
Koolen, W.~M.; Warmuth, M.~K.; Kivinen, J.; et~al. 2010.
\newblock Hedging Structured Concepts.
\newblock In \emph{COLT}, 93--105.

\bibitem[{Krishnasamy et~al.(2018)Krishnasamy, Akhil, Arapostathis, Sundaresan,
  and Shakkottai}]{krishnasamy2018augmenting}
Krishnasamy, S.; Akhil, P.; Arapostathis, A.; Sundaresan, R.; and Shakkottai,
  S. 2018.
\newblock Augmenting max-weight with explicit learning for wireless scheduling
  with switching costs.
\newblock \emph{IEEE/ACM Transactions on Networking}, 26(6): 2501--2514.

\bibitem[{Leung(1989)}]{leung1989bin}
Leung, J.~Y. 1989.
\newblock Bin packing with restricted piece sizes.
\newblock \emph{Information Processing Letters}, 31(3): 145--149.

\bibitem[{Li and Li(2020)}]{li2020leveraging}
Li, Y.; and Li, N. 2020.
\newblock Leveraging predictions in smoothed online convex optimization via
  gradient-based algorithms.
\newblock \emph{arXiv preprint arXiv:2011.12539}.

\bibitem[{Li, Qu, and Li(2020)}]{li2020online}
Li, Y.; Qu, G.; and Li, N. 2020.
\newblock Online optimization with predictions and switching costs: Fast
  algorithms and the fundamental limit.
\newblock \emph{IEEE Transactions on Automatic Control}.

\bibitem[{Lin et~al.(2012{\natexlab{a}})Lin, Liu, Wierman, and
  Andrew}]{lin2012online}
Lin, M.; Liu, Z.; Wierman, A.; and Andrew, L.~L. 2012{\natexlab{a}}.
\newblock Online algorithms for geographical load balancing.
\newblock In \emph{2012 international green computing conference (IGCC)},
  1--10. IEEE.

\bibitem[{Lin et~al.(2012{\natexlab{b}})Lin, Wierman, Andrew, and
  Thereska}]{lin2012dynamic}
Lin, M.; Wierman, A.; Andrew, L.~L.; and Thereska, E. 2012{\natexlab{b}}.
\newblock Dynamic right-sizing for power-proportional data centers.
\newblock \emph{IEEE/ACM Transactions on Networking}, 21(5): 1378--1391.

\bibitem[{Mattingley, Wang, and Boyd(2011)}]{mattingley2011receding}
Mattingley, J.; Wang, Y.; and Boyd, S. 2011.
\newblock Receding horizon control.
\newblock \emph{IEEE Control Systems Magazine}, 31(3): 52--65.

\bibitem[{Sellke(2020)}]{sellke2020chasing}
Sellke, M. 2020.
\newblock Chasing convex bodies optimally.
\newblock In \emph{SODA}, 1509--1518. SIAM.

\bibitem[{Shalev-Shwartz et~al.(2011)}]{shalev2011online}
Shalev-Shwartz, S.; et~al. 2011.
\newblock Online learning and online convex optimization.
\newblock \emph{Foundations and trends in Machine Learning}, 4(2): 107--194.

\bibitem[{Srebro, Sridharan, and Tewari(2011)}]{srebro2011universality}
Srebro, N.; Sridharan, K.; and Tewari, A. 2011.
\newblock On the universality of online mirror descent.
\newblock In \emph{NIPS}, 2645--2653.

\bibitem[{Thein(2014)}]{thein2014apache}
Thein, K. M.~M. 2014.
\newblock Apache kafka: Next generation distributed messaging system.
\newblock \emph{International Journal of Scientific Engineering and Technology
  Research}, 3(47): 9478--9483.

\bibitem[{Zinkevich(2003)}]{zinkevich2003online}
Zinkevich, M. 2003.
\newblock Online convex programming and generalized infinitesimal gradient
  ascent.
\newblock In \emph{ICML}, 928--936.

\end{thebibliography}
\bibliographystyle{icml2022}

\newpage
\appendix
\onecolumn
\appendix
\onecolumn
\section*{Appendix}

\section{Computation Infrastructure}\label{sec:computation-infrastructure}
All the experiments were run on instances with 8 CPUs using 2nd generation Intel Xeon Platinum 8000 series processor (Skylake-SP or Cascade Lake) with a sustained all core Turbo CPU clock speed of up to 3.6 GHz. All algorithms do not require GPU to run. The implementation will be made available when accepted.

\section{Societal Impact}\label{sec:societal-impact}
The idea of smoothed online combinatorial optimization is not restricted to distributed streaming systems. Anything with a switching cost can be benefited from the study of smoothed online combinatorial optimization, including public policy with a switching cost~\cite{farrell2007coordination}, medication and wireless scheduling problems~\cite{krishnasamy2018augmenting}, where both of these can impact the process of policy making and scheduling algorithms.
Including the distributed streaming system problem, these are all applications of smoothed online combinatorial optimization that can lead to change of the current algorithm design in our daily life with impact to the society.

\section{Proofs of Theorem~\ref{thm:finite-proactive-planning} and Theorem~\ref{thm:dynamic-proactive-planning}}
\segmentedRegret*
\begin{proof}
For simplicity of the proof, we define function $g(x,y,z)$ as follows:
\begin{align*}
    g(x,y,z) := f(x,y) + d(x,z)
\end{align*}
which includes both the cost from the cost function $f$ and the switching cost $d$.

Let $\{ \decision'_{s} \}_{s \in \{t,t+1,\cdots, t+S-1 \}}$ be the optimal solutions when the full information of the cost function parameters $\{ \parameter_{s} \}_{s \in \{t,t+1,\cdots, t+S-1 \}}$ is given. Let $\{ \decision_{s} \}_{s \in \{t,t+1,\cdots, t+S-1 \}}$ be the optimal solutions using the predicted parameters $\{ \parameter_{s}^{(t)} \}_{s \in \{t,t+1,\cdots, t+S-1 \}}$. Without loss of generality, we let $\decision'_{t-1} = \decision_{t-1}$ to be the same initial decision at the time step $t-1$.
We have:
\begin{align}
    & \text{Reg}_{t}^{t+S-1}(\decision_{t-1}) \nonumber \\
    = & \left( \sum\limits_{s=t}^{t+S-1} g(\decision_{s}, \parameter_{s} , \decision_{s-1}) - g(\decision'_{s}, \parameter_{s} , \decision'_{s-1}) \right) \nonumber \\
    = & \left( \sum\limits_{s = t}^{t+S-1} g(\decision_{s}, \parameter_{s}, \decision_{s-1}) - g(\decision_{s}, \parameter^{(t)}_{s}, \decision_{s-1}) \right) + \left( \sum\limits_{s = t}^{t+S-1} g(\decision_{s}, \parameter^{(t)}_{s}, \decision_{s-1}) - g(\decision'_{s}, \parameter^{(t)}_{s}, \decision'_{s-1}) \right) \nonumber \\
    & \quad \quad \quad \quad \quad + \left( \sum\limits_{s = t}^{t+S-1} g(\decision'_{s}, \parameter^{(t)}_{s} , \decision'_{s-1}) - g(\decision'_{s}, \parameter_{s}, \decision'_{s-1}) \right) \label{eqn:telescoping-sum} \\
    \leq & \sum\limits_{s = t}^{t+S-1} L \norm{\parameter_{s} - \parameter^{(t)}_{s}} + 0 + \sum\limits_{s = t}^{t+S-1} L \norm{ \parameter^{(t)}_{s} - \parameter_{s} } \nonumber \\
    = & 2 L \sum\limits_{s = t}^{t+S-1} \norm{\parameter_{s} - \parameter^{(t)}_{s}} \nonumber \\
    \leq & 2 L \sum\limits_{s = t}^{t+S-1} \error^{(t)}_{s}  \nonumber
\end{align}
The first term in Eq.~\eqref{eqn:telescoping-sum} can be bounded by (similar the third term):
\begin{align*}
    g(\decision_s, \parameter_s, \decision_{s-1}) - g(\decision_s, \parameter_s^{(t)}, \decision_{s-1}) & = f(\decision_s, \parameter_s, \decision_{s-1}) + d(\decision_s, \decision_{s-1}) - f(\decision_s, \parameter_s^{(t)}, \decision_{s-1}) - d(\decision_s, \decision_{s-1}) \\
    & = f(\decision_s, \parameter_s, \decision_{s-1}) - f(\decision_s, \parameter_s^{(t)}, \decision_{s-1}) \leq L \norm{\parameter_s - \parameter_s^{(t)}}
\end{align*}
The second term in Eq.~\eqref{eqn:telescoping-sum} is non-positive because the optimality of the sequence $\{ \decision_{s} \}_{s \in \{t,t+1,\cdots, t+S-1 \}}$ when using the predictions as the parameters, i.e.,
\begin{align}
    \sum\limits_{s = t}^{t+S-1} g(\decision_{s}, \parameter^{(t)}_{s}, \decision_{s-1}) \leq \sum\limits_{s = t}^{t+S-1} g(\decision^*_{s}, \parameter^{(t)}_{s}, \decision^*_{s-1}) \nonumber
\end{align}

\end{proof}

\totalRegret*

\begin{proof}
In the offline setting, given all the traffic up to time $T$, we can solve the optimization problem in Eq.~\eqref{eqn:finite-optimization-problem} to get the optimal solution $\decision^*$. We use $\text{cost}(\decision^*, \parameter)$ to denote the optimal offline cost.

On the other hand, we assume that Algorithm~\ref{alg:infinite-proactive-planning} runs with $I$ restarts and each restart runs $S_i$ time steps using the predictions to plan ahead for each $i \in [I]$.
Let $T_i = \sum _{j=1}^{i-1} S_j + 1$ be the start time of the $i$-th planning window.
We can split the decisions into chunks --- $\{ \decision_{T_i + s} \}_{s \in \{0, 1, \cdots, S_i-1\} }$ for each $i \in [I]$ that correspond to the decisions obtained in the $i$-th planning window.

Now we would like to compare the cost of the offline optimal solution $\{ \decision^*_t \}_{t \in [T]}$ with the online solution $\{ \decision_t \}_{t \in [T]}$ within the $i$-th chunk $\{T_i, T_i + 1, \cdots, T_i + S_i - 1 \}$.
Since the initial point $\decision^*_{T_i-1}$ of the offline optimal solution and the initial point $\decision_{T_i-1}$ of the online solution are different, we cannot directly apply the result in Theorem~\ref{thm:finite-proactive-planning} to bound the regret.

To resolve the misalignment, we additionally define $\{ \decision'_t \}_{t \in \{T_i, T_i + 1, \cdots, T_i + S_i - 1 \}}$ to be a {\it new} offline optimal solution starting from $T_i$ till $T_i + S_i -1$ with $\decision'_{T_i-1} = \decision_{T_i-1}$ being the initial point. $\decision'_t$ serves as an intermediate to link $\decision^*_t$ and $\decision_t$. Compare to this new offline solution with the same initial decision, the corresponding regret becomes:
\begin{align}
    \text{Reg}_{T_i}^{T_i+S_i-1} = \text{Reg}_{T_i}^{T_i+S_i-1}(\decision_{T_i-1}) \coloneqq \sum\limits_{t = T_i}^{T_i+S_i-1} \left( f(\decision_t, \parameter_t) + d(\decision_t, \decision_{t-1}) \right) - \sum\limits_{t = T_i}^{T_i+S_i-1} \left( f(\decision'_t, \parameter_t) + d(\decision'_t, \decision'_{t-1}) \right)
    \label{eqn:new-regret-definition}
\end{align}

Therefore, we can write:
\begin{align}
    & \sum\limits_{t = T_i}^{T_i+S_i-1} \left( f(\decision_t, \parameter_t) + d(\decision_t, \decision_{t-1}) \right) \label{eqn:online-total-cost} \\
    = ~& \text{Reg}_{T_i}^{T_i+S_i-1} + \sum\limits_{t = T_i}^{T_i+S_i-1} \left( f(\decision'_t, \parameter_t) + d(\decision'_t, \decision'_{t-1}) \right) \label{eqn:partial-offline-optimal-full-information} \\
    \leq ~& \text{Reg}_{T_i}^{T_i+S_i-1} + f(\decision^*_{T_i}, \parameter_{T_i}) + d(\decision^*_{T_i}, \decision_{T_i-1}) + \sum\limits_{t = T_i+1}^{T_i+S_i-1} \left( f(\decision^*_t, \parameter_t) + d(\decision^*_t, \decision^*_{t-1}) \right) \label{eqn:initialized-global-optimal-solution-chunk} \\
    \leq ~& \text{Reg}_{T_i}^{T_i+S_i-1} + B + f(\decision^*_{T_i}, \parameter_{T_i}) + d(\decision^*_{T_i}, \decision^*_{T_i-1}) + \sum\limits_{t = T_i+1}^{T_i+S_i-1} \left( f(\decision^*_t, \parameter_t) + d(\decision^*_t, \decision^*_{t-1}) \right) \label{eqn:initialized-global-optimal-solution-chunk-different} \\
    = ~& \text{Reg}_{T_i}^{T_i+S_i-1} + B + \sum\limits_{t = T_i}^{T_i+S_i-1} \left( f(\decision^*_t, \parameter_t) + d(\decision^*_t, \decision^*_{t-1}) \right) \label{eqn:initialized-global-optimal-solution-chunk-final}
\end{align}
First, from Eq.~\eqref{eqn:online-total-cost} to Eq.~\eqref{eqn:partial-offline-optimal-full-information} is by the definition of $\text{Reg}_{T_i}^{T_i+S_i-1}$ in Eq.~\eqref{eqn:new-regret-definition}.
Second, Eq.~\eqref{eqn:partial-offline-optimal-full-information} to Eq.~\eqref{eqn:initialized-global-optimal-solution-chunk} is due to the optimality of $\decision'_t$:
\begin{align*}
    \{ \decision'_t \}_{t \in \{T_i, T_i + 1, \cdots, T_i + S_i - 1 \}} = \arg\min_{\boldsymbol y} \sum\limits_{t = T_i}^{T_i+S_i-1} \left( f(\boldsymbol y_t, \parameter_t) + d(\boldsymbol y_t, \boldsymbol y_{t-1}) \right), \text{ where } \boldsymbol y_{T_i-1} = \decision_{T_i-1}
\end{align*}
Therefore, plugging in the original optimal solution $\decision^*$ results in a larger cost in Eq.~\eqref{eqn:initialized-global-optimal-solution-chunk}.

Lastly, Eq.~\eqref{eqn:initialized-global-optimal-solution-chunk} and Eq.~\eqref{eqn:initialized-global-optimal-solution-chunk-different} only differ by the initial point at time step $T_i$, where  Eq.~\eqref{eqn:initialized-global-optimal-solution-chunk} uses $\decision_{T_i-1}$ and Eq.~\eqref{eqn:initialized-global-optimal-solution-chunk-different} uses $\decision^*_{T_i-1}$. Thus the difference is bounded by the maximal switching cost $B$.

We can reorganize the inequality in Eq.~\eqref{eqn:initialized-global-optimal-solution-chunk-final} to get:
\begin{align*}
    & \sum _{t = T_i}^{T_i+S_i-1} \left( f(\decision_t, \parameter_t) + d(\decision_t, \decision_{t-1}) \right) - \sum _{t = T_i}^{T_i+S_i-1} \left( f(\decision^*_t, \parameter_t) + d(\decision^*_t, \decision^*_{t-1}) \right) \nonumber \\
    \leq ~& \text{Reg}_{T_i}^{T_i+S_i-1}(\decision_{T_i-1}) + B \nonumber \\
    \leq ~& 2 L \sum _{s=T_i}^{T_i+S_i-1} \error_s^{(t)} + B \nonumber \\
    = ~& 2B
\end{align*}
where the last inequality is by the choice of the dynamic planning window $S_i$ such that $2 L \sum _{s=T_i}^{T_i+S_i-1} \error_s^{(t)} \leq B$.
Lastly, we can take summation over all the $i \in [I]$ to get:
\begin{align}
    \text{Reg}_T &= \sum\limits_{t = 1}^{T} \left( f(\decision_t, \parameter_t) + d(\decision_t, \decision_{t-1}) \right) - \sum\limits_{t = 1}^{T} \left( f(\decision^*_t, \parameter_t) + d(\decision^*_t, \decision^*_{t-1}) \right) \nonumber \\
    & = \sum\limits_{i=1}^I \left( \sum\limits_{t = T_i}^{T_i+S_i-1} \left( f(\decision_t, \parameter_t) + d(\decision_t, \decision_{t-1}) \right) - \sum\limits_{t = T_i}^{T_i+S_i-1} \left( f(\decision^*_t, \parameter_t) + d(\decision^*_t, \decision^*_{t-1}) \right) \right) \nonumber \\
    & \leq \sum\limits_{i=1}^{I} 2B \nonumber \\
    & = 2BI \nonumber
\end{align}
\end{proof}

\section{Proof of Corollary~\ref{coro:regret-bound}}
\regretBoundPoly*
To prove Corollary~\ref{coro:regret-bound}, we need the following lemmas:

\begin{lemma}\label{lemma:recursion}
Given any fixed $0 \leq \alpha$ and the following recursive formula:
\begin{align*}
    T_1 = & ~ 1 \\
    T_{i+1} \geq & ~ T_{i} + A \cdot T_{i}^\alpha, \forall i \geq 1 .
\end{align*}
We can show:
\begin{align*}
    T_{i} \geq \begin{cases}
    c \cdot i^{\beta} & \text{if $\alpha < 1$} \\
    (A+1)^{i-1} & \text{if $\alpha = 1$} \\
    (A+1)^{\alpha^{(i-2)}} & \text{if $\alpha > 1$}
    \end{cases}
\end{align*}
where $\beta = \frac{1}{1-\alpha}$ if $\alpha < 1$. The constant $c \in \R_{\geq 0}$ satisfies $c \leq (\frac{e^\beta}{A})^{1 - \alpha} = e A^{\alpha - 1}$ and $c \leq 1$.
\end{lemma}
We prove three different cases separately.
\begin{itemize}
    \item {\bf Case 1 ($\alpha < 1$)} (this is deferred to the end).
    \item {\bf Case 2 ($\alpha = 1$)}.
    \item {\bf Case 3 ($\alpha > 1$)}.
\end{itemize}

\begin{proof}[Proof of Case 2]
    The recursive formula reduces to $T_{i+1} \geq (A+1) T_i$, where we can easily show that $T_i \geq (A+1)^{i-1}$.
\end{proof}

\begin{proof}[Proof of Case 3]
    The recursive formula can be written as $T_{i+1} \geq T_i^\alpha$ and $T_2 \geq A+1$. Thus we can simply unroll the recursion to get
\begin{align}
    T_i \geq T_{i-1}^\alpha \geq T_{i-2}^{\alpha^2} \geq \cdots \geq T_2^{\alpha^{(i-2)}} = (A+1)^{\alpha^{(i-2)}} \nonumber
\end{align}
\end{proof}

\begin{proof}[Proof of Case 1]
We prove by induction.

{\bf Base case: }
Since $c \leq 1$, the base case is automatically satisfied by $1 = T_1 \geq c = c \cdot 1^\beta$.

{\bf Inductive step:}
By induction, assume $T_i \geq c \cdot i^\beta$. By our choice of $c$, we can see that $A c^{\alpha-1} \geq e^\beta$, which implies:
\begin{align}\label{eq:rewrite_Aci}
     A c^{\alpha-1} i^{\alpha \beta} \geq e^\beta \cdot i^{\alpha \beta} = e^\beta \cdot i^{\beta-1} 
\end{align}
where the second step follows from $\alpha \beta = \beta-1$.

Therefore, we can lower bound $T_{i+1}$ by:
\begin{align*}
    T_{i+1} 
    \geq & ~ T_i + A T_i^\alpha \\
    \geq & ~ (c \cdot i^{\beta}) + A (c \cdot i^{\beta})^{\alpha} & \text{~by~} T_i \geq c i^{\beta}\\
    = & ~ c \cdot ( i^\beta + A c^{\alpha-1} i^{\alpha \beta} )  \\
    \geq & ~ c \cdot ( i^\beta + e^\beta i^{\beta - 1} ) & \text{~by~Eq.~\eqref{eq:rewrite_Aci}} \\
    \geq & ~ c \cdot (i+1)^\beta, & \text{~by~Lemma~\ref{lem:binomial}}
\end{align*}
where we can apply Lemma~\ref{lem:binomial} because $\beta = \frac{1}{1 - \alpha} \geq 1$ for all $\alpha \in [0,1)$.
\end{proof}

\begin{lemma}\label{lem:binomial}
\begin{align}
    x^k + e^k x^{k-1} \geq (x+1)^k \quad ~\forall x \geq 1, k \geq 1
\end{align}
\end{lemma}
\begin{proof}
Define a function $g(x,k) = x^k + e^k x^{k-1} - (x+1)^k$. We can check that $g(x,1) = x + e - (x+1) > 0$. Next, we show that $g(x,k)$ is an increasing function in $k$ when $x \geq 1$. Notice that the derivative $\frac{d g}{d k}$ can be written as:
\begin{align}
    \frac{d g}{d k} \mid_{x,k} = & \log x \cdot x^k + e^k x^{k-1} + \log x \cdot e^k x^{k-1} - \log (x+1) \cdot (x+1)^k \nonumber \\
    = & \log x \cdot x^k + e^k x^{k-1} + \log x \cdot e^k x^{k-1} - \log x \cdot (x+1)^k - \log (\frac{1+x}{x}) \cdot (x+1)^k \nonumber \\
    = & \log x \cdot (x^k + e^k x^{k-1} - (x+1)^k) + e^k x^{k-1} - \log (1 + \frac{1}{x}) \cdot (x+1)^k \nonumber \\
    \geq & \log x \cdot g(x,k) + \left( e^k x^{k-1} - \frac{1}{x} \cdot (x+1)^k \right) \label{eqn:binomial-proof-derivative-bound}
\end{align}
where the last inequality is due to $\log (1 + \frac{1}{x}) \leq \frac{1}{x}$.

The second term in Eq.~\eqref{eqn:binomial-proof-derivative-bound} can be written as:
\begin{align}
    & e^k x^{k-1} - \frac{1}{x} \cdot (x+1)^k = \frac{1}{x} \left( (ex)^{k} - (x+1)^k \right) > 0 \label{eqn:binomial-proof-derivative-bound-second-term}
\end{align}
which is always satisfied because $ex > x+1 ~\forall x \geq 1$.

Therefore, Eq.~\eqref{eqn:binomial-proof-derivative-bound} and Eq.~\eqref{eqn:binomial-proof-derivative-bound-second-term} together guarantee that if the value $g(x,k) \geq 0$, then its derivative is positive $\frac{d g}{d k} \mid_{x,k} > 0$ because every term in Eq.~\eqref{eqn:binomial-proof-derivative-bound} is positive.
So now we have $g(x,1) > 0$ and the derivative $\frac{d g}{d k} \mid_{x,k} > 0$ if $g(x,k) \geq 0$.

Lastly, we just need to ensure that the function is always non-negative.
Given fixed $x$, define $U = \{ k > 1 \mid g(x,k) < 0 \}$.
We will prove by contradiction by assuming $U$ is non-empty.
Given that $U$ is not empty, we can choose $u = \inf \{k: k \in U \} $. By the continuity of function $g$, $g(x,u) \leq 0$. Since $g(x,1) > 0$ and the continuity of $g$, we can find $g(x, 1 + \epsilon) > 0$ for all $\epsilon \in B(0,r)$ in a small open ball. Thus $u \geq 1 + \epsilon > 1$. Now by the mean value theorem applied on $g(x,1) > 0$ and $g(x,u) \leq 0$, we can find a value $v \in (1,u)$ such that $g(x,v) = \frac{g(x,u) - g(x,1)}{u - 1} < 0$. However, we have proven that if $g(x,k) \geq 0$ then we know $\frac{d g}{d k} \mid_{x,k} > 0$. Since we have $g(x,v) < 0$, this implies $g(x,v) < 0$ as well with $v \in (1,u)$ and thus $v \in U$, which contradicts to the definition of $u$, i.e., the infimum of the set $U$.
Thus the assumption that $U$ is non-empty is incorrect. We conclude that $U$ is empty. Thus for any given $x$, we have $g(x,k) \geq 0$ for all $k$, which implies the original inequality.

\end{proof}





Now we are ready to prove Corollary~\ref{coro:regret-bound}.
\begin{proof}[Proof of Corollary~\ref{coro:regret-bound}]
First, let $S_i$ denote the size of the $i$-th planning window in Algorithm~\ref{alg:infinite-proactive-planning} for each $i \in [I]$.
Let $T_i = \sum _{j=1}^{i-1} S_j + 1$ denote the start time of the $i$-th planning part.

In the $i$-th iteration of Algorithm~\ref{alg:infinite-proactive-planning} starting at time $T_i$, $S_i$ is chosen such that $S_i$ is the largest integer\footnote{We need $B \geq \error_{T_i}^{(T_i)}$ to ensure that we can at least choose $S_i \geq 1$. In the extreme case where $B < \error_{T_i}^{(T_i)}$, it implies that the uncertainty is too large while the switching cost is relatively small. Thus it is ideal to re-plan every time step because switching is cheap. The analysis of balancing switching cost and future planning does not apply.} satisfying $2 L \sum\limits_{s=T_i}^{T_i+S_i-1} \error_s^{(T_i)} \leq B$.
This implies $2 L \sum\limits_{s=T_i}^{T_i+S_i} \error_s^{(T_i)} > B$ and we can estimate $S_i$ by:
\begin{align}
    B & < 2L \sum\limits_{s=1}^{S_i+1} \epsilon_{T_i+s-1}^{(T_i)} \leq 2L D \sum\limits_{s=1}^{S_i+1} \frac{s^a}{T_i^b} \leq \frac{2LD}{a+1} \frac{(S_i+2)^{a+1}}{T_i^b} \nonumber
\end{align}
for some constant $D > 0$. This suggests:
\begin{align*} 
    & (\frac{a+1}{2D})^{\frac{1}{a+1}} (\frac{B}{L})^{\frac{1}{a+1}} T_i^{\frac{b}{a+1}} \leq S_i + 2 \leq 3 S_i, & A T_i^{\frac{b}{a+1}} \leq S_i
\end{align*}
where $A = \frac{1}{3} (\frac{a+1}{2D})^{\frac{1}{a+1}} (\frac{B}{L})^{\frac{1}{a+1}} = \Theta((\frac{B}{L})^{\frac{1}{a+1}})$ is a constant dependent on the maximal switching cost $B$ and the Lipschitzness $L$.

Therefore, we have
\begin{align}
    T_1 = 1, \quad T_{i+1} = T_i + S_i \geq T_i + A T_i^{\frac{b}{a+1}} \nonumber
\end{align}
where Lemma~\ref{lemma:recursion} can be applied to get:
\begin{align*}
    T_i &\geq \begin{cases}
    c i^{\frac{a+1}{a+1-b}} & \text{if $b < a + 1$} \\
    (A+1)^{i-1} & \text{if $b = a + 1$} \\
    (A+1)^{\frac{b}{a+1}^{(i-2)}} & \text{if $b > a + 1$}
    \end{cases}
\end{align*}
with the choice of the constant $c = \min(1, e A^{\frac{a+1-b}{a+1}})$.
Lastly, since $T_{I} \leq T$, we can bound the total iteration $I$ by:
\begin{align*}
    T \geq T_{I} \geq \begin{cases}
    c I^{\frac{a+1}{a-b+1}} & \text{if $b < a+1$} \\
    (A+1)^{I-1} & \text{if $b=a+1$} \\
    (A+1)^{\frac{b}{a+1}^{(I-2)}} & \text{if $b > a + 1$}
    \end{cases}
\end{align*}
which gives: 
\begin{align*}
    I \leq & \begin{cases}
    (\frac{T}{c})^{\frac{a-b+1}{a+1}} \\
    \log_{A+1} T + 1 \\
    \log_{A+1} \log_{\frac{b}{a+1}} T + 2
    \end{cases}
    = \begin{cases}
    O(T^{1 - \frac{b}{a+1}}) & \text{if $b < a+1$} \\
    O(\log T) & \text{if $b = a+1$} \\
    O(\log \log T) & \text{if $b > a+1$}
    \end{cases} \\
\end{align*}



By applying Theorem~\ref{thm:dynamic-proactive-planning} and substituting the total number of iterations $I$ by the above inequality, we get:
\begin{align*}
    \text{Reg}_T & \leq \Theta(BI) \nonumber = \begin{cases}
    O(T^{1 - \frac{b}{a+1}}) & \text{if $b < a+1$} \\
    O(\log T) & \text{if $b = a+1$} \\
    O(\log \log T) & \text{if $b > a+1$}
    \end{cases}
\end{align*}
\end{proof}




\section{Proof of Corollary~\ref{coro:lower-bound}}
\lowerBound*

\begin{proof}
Let $\error^{(t)}_{s} = \frac{1}{t^b} = O(\frac{s^a}{t^b})$ for all $t, s \in \N$ with $\frac{1}{t^b} < \frac{1}{2}$.
We construct a sequence of one-dimensional incoming traffic $\parameter_t = \begin{cases} 
\frac{1}{2} + \frac{1}{t^b} \\
\frac{1}{2} - \frac{1}{t^b}
\end{cases}$ and a one-dimensional feasible set $\decisionset = \{0, 1\}$.
The prediction given by the predictive model is $\parameter^{(t)}_s = \frac{1}{2}$ for all $t, s \in \N$, whose predictive error satisfies the bound $\norm{\parameter_s - \parameter^{(t)}_s} \leq \frac{1}{t^b} = \error^{(t)}_{s}$.
Assume that the cost function is defined by $f(\decision, \parameter) = L \norm{\decision - \parameter}$ and there is no switching cost $d(\boldsymbol x, \boldsymbol y) = 0$.

Under this construction, if all the incoming traffics are given in advance, the optimal cost within $T$ time steps is:
\begin{align}
    & L \sum\limits_{i=1}^T \left(\frac{1}{2} - \frac{1}{t^b}\right) = \frac{LT}{2} - L \sum\limits_{i=1}^T t^{-b} \leq \begin{cases}
    \frac{LT}{2} - \frac{1}{1-b} T^{1-b} & \text{if $b < 1$} \\
    \frac{LT}{2} - \log T & \text{if $b = 1$} \\
    \frac{LT}{2} - \Theta(1) & \text{if $b > 1$}
\end{cases} \nonumber
\end{align}
where we can just choose $\decision_t = 1$ if $\parameter_t$ is closer to $1$ and $0$ otherwise.

On the other hand, if the incoming traffics are not given in advance, any decision made at time step $t$ produces cost $L(\frac{1}{2} + \frac{1}{t^b})$ with probability $\frac{1}{2}$ and cost $L(\frac{1}{2} - \frac{1}{t^b})$ with probability $\frac{1}{2}$, which gives expected cost $\frac{L}{2}$ and a cumulative cost $\frac{LT}{2}$. 
Therefore, the expected cumulative regret is at least $ \begin{cases}
\Theta(T^{1-b}) & \text{if $b < 1$} \\
\Theta(\log{T}) & \text{if $b = 1$} \\
\Theta(1) & \text{if $b > 1$}
\end{cases}
$.
\end{proof}

\end{document}